\documentclass{article}

\newif\ifarxiv
\arxivfalse %
\arxivtrue %

\ifarxiv %
\usepackage[final]{hyperref}
\usepackage{geometry}
\usepackage[
backend=bibtex,
style=alphabetic,
sorting=ynt,
maxnames=100,
maxcitenames=4,
natbib=true,
]{biblatex}
\usepackage{xcolor}

\definecolor{mydarkblue}{rgb}{0,0.08,0.45}
\definecolor{mycitationcolor}{rgb}{0,0.08,0.45}
\definecolor{mylinkcolor}{rgb}{1,0.0,0} %
\definecolor{mylinkcolor}{rgb}{0,0.08,0.85}%
\hypersetup{ %
	pdfborder=0 0 0,
	colorlinks=true,
	linkcolor=mylinkcolor,
	citecolor=mycitationcolor,
	filecolor=mydarkblue,
	urlcolor=mydarkblue,
}

\usepackage{booktabs} %
\usepackage[algo2e,ruled,vlined]{algorithm2e}
\usepackage{algorithm}%
\usepackage{algpseudocode}%

\usepackage{hyperref}

\fi

\usepackage{amsmath,amssymb,amsfonts,amsthm}
\usepackage{graphicx}

\usepackage{hyperref}

\usepackage{color}
\usepackage{xcolor}
\definecolor{skyblue}{rgb}{0.529,0.808,0.922}
\definecolor{LightCyan}{rgb}{0.88,1,1}
\definecolor{LightLightCyan}{rgb}{1.0,1.0,1.0}
\definecolor{Gray}{gray}{0.85}
\definecolor{DarkGray}{gray}{0.95}
\usepackage{colortbl}

\def\metlearntitle{
Sample Efficient Linear Meta-Learning by Alternating Minimization
}

\usepackage{xspace}
\def\altmin{{MLLAM \xspace}}
\def\altmins{{MLLAMS \xspace}}

\def\cL{{\mathcal{L}}}

\def\xij{x^{(i)}_j}
\def\yij{y^{(i)}_j}
\def\epsi{\varepsilon^{(i)}}
\def\epsij{\varepsilon^{(i)}_j}
\def\u{u}
\def\v{V}
\def\t{t}
\def\vi{v^{(i)}}

\def\dv{{H}}
\def\dvi{h^{(i)}}
\def\cH{\mathcal{H}}
\def\hcH{\widehat{\mathcal{H}}}
\def\cV{\mathcal{W}}
\def\cVi{\cV^{-1}}
\def\cVh{\cV^{\frac12}}
\def\cVih{\cV^{-\frac12}}
\def\VV{{W}}

\def\VVih{{\VV}^{-\frac{1}2}}

\def\inftyone{{\infty,2}}
\def\v{V}
\def\hvi{\widehat{v}^{(i)}}
\def\hvij{{\widehat{v}^{(i)}_j}}
\def\hu{\widehat{u}}
\def\hU{\widehat{U}}
\def\pu{{u^+}}
\def\pU{{U^+}}
\def\iU{{U_{\mathrm{init}}}}
\def\ou{u^*}
\def\ov{V^*}
\def\oU{U^*}
\def\ovi{v^*^{(i)}}
\def\owi{\widetilde{v}^{*(i)}}
\def\ow{\widetilde{V}^*}
\def\ovi{v^{*(i)}}
\def\Si{S^{(i)}}

\def\zi{z^{(i)}}
\def\gi{g^{(i)}}
\def\hgj{{\widehat{g}_j}}
\def\cA{\mathcal{A}}
\def\hcA{{\widehat{\mathcal{A}}}}
\def\cI{\mathcal{I}}
\def\cE{\mathcal{E}}
\def\cN{\mathcal{N}}
\def\bS{\mathbb{S}}
\def\inv{{\dagger}}
\def\invert{{-1}}
\newcommand{\Ip}[2]{\left\langle#1, #2\right\rangle}
\def\tr{\mathrm{tr}}
\def\Id{\mathbf{I}}
\def\Ind{\mathbb{I}}
\def\bR{\mathbb{R}}
\def\Ord{{O}}
\def\tOrd{\widetilde{\Ord}}
\def\bE{\mathbb{E}}
\def\eigmax{\lambda_1}
\def\eigmin{\lambda_r}

\def\sigmin{\sigma_r}

\def\SNRn{{\frac{\sigma}{\sqrt{{\eigmin^*} }}}}
\def\SNRnl{{{\sigma}/{\sqrt{{\eigmin^*} }}}}

\def\SNRre{{\frac{\sigma}{\sqrt{{\eigmin^*}}}}}

\def\cT{{\mathcal{T}}}
\def\cP{{\mathcal{P}}}
\def\hcT{{\widehat{\mathcal{T}}}}
\def\hatt{{\widehat{t}}}

\newtheorem{thm}{Theorem}
\newtheorem{coro}[thm]{Corollary}

\newtheorem{assume}{Assumptions}

\newtheorem{propo}{Proposition}[section]
\newtheorem{lemma}[propo]{Lemma}

\usepackage{subcaption}

\begin{document}
	
	\ifarxiv
	\title{{\metlearntitle}}
	\date{}
	\author{
		Kiran Koshy Thekumparampil$^\dagger$, Prateek Jain$^\ddagger$,  Praneeth Netrapalli$^\ddagger$, Sewoong Oh$^\pm$
		\thanks{Author emails are {\text thekump2@illinois.edu}, \text{prajain@google.com}, {\text{pnetrapalli@google.com}, and \text{sewoong@cs.washington.edu}}. 
		}\\
		\\
		$^\dagger$University of Illinois at Urbana-Champaign,
		$^\ddagger$Google Research, India,\\
		$^\pm$University of Washington, Seattle
	}
	\maketitle
	\else
	\fi
	
	\begin{abstract}
		Meta-learning synthesizes and leverages the knowledge from a given set of tasks to rapidly learn new tasks using very little data. 
		Meta-learning of linear regression tasks, where the regressors lie in a low-dimensional subspace, is an extensively-studied fundamental problem in this domain. 
		However, existing results either guarantee highly suboptimal estimation errors, or require $\Omega(d)$ samples per task (where $d$ is the data dimensionality) thus providing little gain over separately learning each task.
		In this work, we study a simple alternating minimization method (MLLAM), which alternately learns the low-dimensional subspace and the regressors.
		We show that, for a constant subspace dimension %
		MLLAM obtains nearly-optimal estimation error, despite requiring only $\Omega(\log d)$ samples per task. 
		However, the number of samples required per task grows logarithmically with the number of tasks. 
		To remedy this in the low-noise regime, we propose a novel task subset selection scheme that ensures the same strong statistical guarantee as MLLAM, 
		even with {\em bounded} number of samples per task for arbitrarily large number of tasks.
	\end{abstract}

\section{Introduction}
\label{sec:int}

Common real world tasks follow a long tailed distribution where most of the tasks only have a small number of labelled examples 
\cite{wang2017learning}. Collecting more  clean labels is often costly (e.g., medical imaging). 
As each task does not have enough examples to be learned in isolation, meta-learning attempts to 
meta-learn across a large number of tasks by exploiting some structural similarities among those tasks. 
One popular  approach is to learn a  shared representation, where all of those tasks can be solved accurately \cite{sun2017revisiting}. 
Once such a representation has been learnt,  we can rapidly adapt to  new arriving tasks, learning a model with only a few examples. 
Empirical evidences suggest  that this might also explain  recent successes in few-shot supervised learning with optimization based methods like MAML \cite{finn2017model,maml_rep}. 

In this paper, we study the problem of linear meta-representation learning \cite{du2020few,tripuraneni2020provable}, where the goal is to learn a $r$-dimensional linear representation/subspace that is shared by a collection of $t$ linear regression tasks in $d$ dimensions. Each task has $m$ labelled examples. 

We investigate a fundamental question: as the number of tasks grow, can we learn the underlying $r$-dimensional shared representation (subspace) more accurately, and consequently learn more accurate regressors per task? 
The question is important because in general, the number of tasks can be large while a lot of tasks are data starved. Furthermore, in several settings like crowdsourcing or bioinformatics, it might be easier to collect more data for new tasks, instead of collecting more data for the existing tasks. 

Most of the existing work do not provide a satisfactory solution for this fundamental problem. In particular, \citet{du2020few} require $m=\Omega(d)$ samples per task, which is prohibitively large, and in fact with so many samples, one can solve each task in isolation.   The Burer-Monteiro factorization approach of \citet{tripuraneni2020provable} is not able to provide any improvement by increasing the number of tasks, it needs to increase samples per task to $m=\Omega(1/\varepsilon^2 )$ to achieve $\varepsilon$ accuracy. While the method-of-moments approach proposed in \citep{tripuraneni2020provable,kong2020meta} does provide more accurate representation learning with a larger number of tasks,  the method has a highly sub-optimal dependence on the noise variance $\sigma^2$ associated with each task. For example, even when each regression task can be solved exactly with $0$ error, this method will incur a significant error.

\begin{table*}[tbp]
	\centering
	{\caption{Comparison of high-probability error bounds for the distance between the learned ($U$) and the true ($\oU$) subspaces, for linear low-rank meta-learning in $d$ dimensions with $t$ tasks, $m$ samples per task, and noise variance $\sigma^2$. Note that $\tOrd$ and $\widetilde\Omega$ hides $\mathrm{polylog}$ factors in $d$ and $\log\log$ factors in $t$.
	We assume a constant small subspace rank, constant incoherence of tasks, constant magnitude for regressors, and well-conditioned task diversity matrix. Note that non-convex ERM~\citep{du2020few} is a result for general non-linear meta-learning.
		}
	\label{tab:results_altmin_meta}%
	}%
	{\small 
		\begin{tabular}{c   c  c}
			\toprule
			\bfseries Algorithm & \bfseries Error-bound: $\|(\Id - \oU (\oU)^\top)U\|$ 
			& \bfseries Required samples per task \\
			\midrule
			\rowcolor{DarkGray}
			Non-convex ERM~\citep{du2020few} & 
			$\tOrd(\sigma) \sqrt{\frac{t + d}{m\,t}}$ 
			&  $m \geq \widetilde\Omega(d + \log(t))$ \\
			Burer-Monteiro factorization~\citep{tripuraneni2020provable}   & $\tOrd(\sigma)\sqrt{\frac{\max(t, d)}{m\,t}}$ &  $m \geq \widetilde\Omega(\log(t))$ \\
			\rowcolor{DarkGray}
			Method-of-Moments~\citep{tripuraneni2020provable}   & $\tOrd(1 + \sigma^2)\sqrt{\frac{d}{m\,t}}$ &  $\Omega(1)$ \\
			\cellcolor{LightCyan} {\altmin (ours, Theorem~\ref{thm:altmin_logt_informal})}  & \cellcolor{LightCyan}$\tOrd(\sigma)\sqrt{\frac{d}{m\,t}}$ &  \cellcolor{LightCyan} $m \geq \widetilde\Omega((1+\sigma^2)\log t)$
			\\
			\cellcolor{skyblue} {\altmins (ours, Theorem~\ref{thm:altmin_subset_informal})}  & \cellcolor{skyblue}$\tOrd(\sigma)\sqrt{\frac{d}{m\,t}}$ &  \cellcolor{skyblue} $m \geq \widetilde\Omega(1+\sigma^2 \log(t))$
			\\
			Lower-bound~\citep{tripuraneni2020provable}   & $\Omega(\sigma)\sqrt{\frac{d}{m\,t}}$ &  $\Omega(1)$
			\\
			\bottomrule
		\end{tabular}
	}
\end{table*}

\noindent
{\bf Contributions.} 
In this paper, we propose the first efficient  approach for linear meta-learning with provable guarantees that  achieves nearly optimal error rate. According to a Frobenius norm error metric, our bound matches a fundamental lower bound. %
Our first algorithm \altmin is based on alternating minimization, 
inspired by a long line of successes in matrix completion and matrix sensing \cite{jain2013low}. 
Assuming constant dimensionality of the representation, \altmin requires $m=\Omega(\log t + \log\log(1/\varepsilon))$ samples per task to achieve an accuracy of $\varepsilon$ when we have $t$ tasks. Our method obtains nearly optimal dependence on the noise variance $\sigma^2$ and the error in  representation learning drops nearly optimally with growing number of tasks, which is a significant improvement over the state-of-the-art. %

However, the number of samples per task ($m$) still grows logarithmically on $t$. To further improve this dependence, we introduce  MLLAMS 
that applies alternating minimization to only a subset of tasks that are well-behaved.
When the noise is sufficiently small with variance $O(1/\log t)$, this further reduces the requirement down to $m=\Omega(\log\log(1/\varepsilon))$. That is, despite fixed  $m$, MLLAMS can estimate each task more accurately. Furthermore, due to our improved rates on estimation of the $r$-dimensional subspace, the  best known rates for prediction error on new tasks also improve significantly. Table~\ref{tab:results_altmin_meta} compares the error and per-task sample complexity of our method against existing state-of-the-art results; see Section~\ref{sec:problem} for details of the problem setting. 

Broadly, our proof structure follows that of existing alternating minimization results \cite{jain2013low, netrapalliphase} by showing iterative refinement of the estimates. However, existing techniques are able to rely on restricted isometry property style properties, which are significantly more difficult to prove in our case. Furthermore, most of the existing works in this literature analyze non-noisy setting where each observation is sampled exactly from an underlying model. However, in this work, we also allow each observation to be corrupted by a white noise, leading to more challenging per-iterate analysis. 
\\

\noindent
\textbf{Notations:} For an whole number $n$, $[n] = \{1,\ldots,n\}$. $\|A\|$ and $\|A\|_F$ denote the spectral and Frobenius norms of the matrix A. $\Ip{A}{B}$ denotes inner produce between two matrices. $A^\inv$ is the Moore-Penrose pseudoinverse and $A^\top$ is the transpose of the matrix $A$. $x \sim \cN(0, \Id_{d \times d})$ means that $x$ is $d$ dimensional standard isotropic Gaussian random vector.

\subsection{Related work} 
\label{sec:related} 

\noindent{\bf Representation learning for meta-learning.} There is a large body of  work in meta-learning from multiple tasks  since the seminal work in learning to learn \cite{Thrun1998learning}, inductive bias learning  \cite{baxter2000model},  and multitask learning \cite{caruana1997multitask}. One popular line of work starting from \cite{intrator1996making,baxter1995learning} is to learn  a low-dimensional representation  for a set of  related tasks and  use the representation  to efficiently train a model for a new arriving task. Recently, these representation learning approaches are gaining more attraction as recent empirical evidence  indicates that the success of other popular  meta-learning approaches such as MAML \cite{finn2017model} is   due to their capability to learn useful low-dimensional representations  \cite{raghu2019rapid}.  

\cite{ando2005framework,rish2008closed,orlitsky2005supervised} address the problem of recovering linear regression  parameters that  lie on an unknown $r$-dimensional subspace $U^*$, where all tasks can be accurately solved. Nuclear-norm minimization approaches are proposed in \cite{argyriou2008convex,harchaoui2012large,amit2007uncovering,pontil2013excess} but they do not provide subspace/generalization error guarantees and suffer from large training time. 

Closest to our work is \cite{tripuraneni2020provable} that analyzes the landscape of the empirical risk with Burer-Monteiro factorization.  It is shown that  $mt = \widetilde{\Omega}( \max\{t,d\}r^4 + \max\{t,d\}r^2\sigma^2/\epsilon^2)$ is sufficient to achieve a rescaled error $(1/\sqrt{r})\| ({\mathbf I}-U^*(U)^*)^T)U \|_F \leq \epsilon$,  where we assumed incoherent  and well-conditioned regression parameters to simplify the condition.   In particular, more tasks do not give any gain beyond a certain point if $m$ is fixed.  Further, it is also required that all tasks are of equal strengths, i.e.,~$\|v^{(i)}\|=\Theta(1)$ for all $i\in[t]$. Another approach is to find the principal directions of a particular 4th moment matrix  \cite{tripuraneni2020provable,kong2020meta}. This only requires $mt=\widetilde{\Omega}( (1+\sigma^2)dr^2/\epsilon^2)$, but the algorithm is inexact;  
the error is bounded away from zero even if there is no noise and sample size is sufficiently large to learn all the parameters. 
This is in a stark contrast with our approach, as illustrated in Figure~\ref{fig:alt_mom_vs_sigma}. 
\citet{du2020few} studies  the  global minimizer of a non-convex optimization in Eq.~\eqref{eq:risk} without analyzing an efficient algorithm to find it. It is  shown that a small generalization error can be achieved if $m = \widetilde\Omega( d)$. 

We also point out a concurrent and independent work \cite{collins2021exploiting}, which proposes and analyzes a slightly different variant (with descent step on $U$) \citep[Algorithm 2]{collins2021exploiting} of our alternating minimization algorithm (Algorithm~\ref{alg:altmin}, \altmin) for a similar linear meta-learning setting. However, this work assumes that the linear meta-learning problem is noiseless, i.e.~$\sigma=0$ (as defined in our Assumption~\ref{assume:linear_meta_problem}), and then it provides a per task sample complexity of $m \geq \widetilde{\Omega}(({{\eigmax^*}/{\eigmin^*}})^2 \,r^3 \log(t))$ and a total sample complexity of $m \t \geq \widetilde{\Omega} (({{\eigmax^*}/{\eigmin^*}})^2 d r^2)$. In contrast, our results are for a more natural noisy setting, and even for the noiseless setting we obtain a tighter per task sample complexity of $m \geq \widetilde{\Omega}(r^2)$ and a total sample complexity of $m \t \geq \widetilde{\Omega} (({{\eigmax^*}/{\eigmin^*}}) (d+r^2) r^2)$ (Corollary~\ref{coro:altmin_subset_informal}) in terms of the condition number $({{\eigmax^*}/{\eigmin^*}})$ and the rank $r$, and our per task complexity does not scale with the number of tasks $t$. \citet{collins2021exploiting} further show that alternating minimization performs better than other baselines for personalized federated learning of neural network classifiers for some datasets.

\medskip\noindent{\bf Matrix sensing.} 
Starting from matrix sensing and completion problems \cite{candesrecht,svp,jain2013low}, recovering a low-rank matrix from linear measurements have been a popular topic of research. Linear meta-learning is a special case of matrix sensing, but with special sensing operator of the form $\mathcal{A}(UV^T)=[A_1(UV^T), \dots, A_{mt}(UV^T)]$ where $A_{ij}(UV^T)=\langle x_{ij}e_i^\top , UV^\top\rangle$. This operator cannot satisfy sensing properties like restricted isometry property, in general because of sparse sensing matrix, so existing matrix sensing results do not apply directly. Furthermore, \cite{jain2013provable,zhong2015efficient} studied a similar problem but their results also require $O(d)$ samples per task, which limits it's applicability to the meta-learning setting where each task has a small number of samples. 
\section{Problem formulation}
\label{sec:problem}
Suppose there are $\t$ $d$-dimensional linear regression tasks, and each of them have $m$ samples. That is for the $i$-th task ($i \in [t]$), we are given $m$ samples $\{(\xij \in \mathbb{R}^d, \yij \in \mathbb{R})\}_{j=1}^m$, where $(\xij, \yij)$ is the $j$-th pair of example and observation. The standard goal is to learn accurate regressors $\widetilde{v}^{*(i)}$ for each of the tasks. However, in the meta-learning setting, all the tasks are related and share a  common but unknown {\em low-dimensional representation} parameterized by $\oU\in \mathbb{R}^{d\times r}$ where $r\ll d$. Here, the goal is to learn $\oU$ and the task specific regressors $v^{(i)}$ s.t. $v^{(i)}$'s are accurate regressor for samples $\{((\oU)^\top\xij \in \bR^r, \yij \in \bR)\}_{j=1}^m$, for $i \in [t]$. 
This is equivalent to finding a set of accurate regressors $\widetilde{v}^{(i)}$'s, which lie in a low-dimensional subspace. 

A natural requirement of the problem is to then learn the tasks accurately with very small number of samples per task, especially for large $t$. As a task specific regressor has only $r$ parameters, given $\oU$, we expect the number of samples per task to depend only on $r$, instead of $d$. Furthermore, the total number of samples $m\cdot \t$ should scale at most linearly with the data dimension $d$. However, simultaneously learning the representation $U$ and the regressors $v^{(i)}$ is challenging.
In fact, since the NP-hard low-rank matrix completion problem~\cite{hardtmeka} can be reduced to the linear meta-learning problem, the latter is NP-hard. Therefore, similar to \citet{tripuraneni2020provable}, we study the problem in the following tractable random design setting.
\begin{assume}\label{assume:linear_meta_problem}
Let $\oU \in \mathbb{R}^{d\times r}$ be an orthonormal matrix. For a task $i \in [t]$, with task specific regressor $\ovi \in \mathbb{R}^r$ and $j$-th example $\xij \sim \cN(0, \Id_{d \times d})$, its observation is\vspace*{-5pt}
\begin{equation}
\label{eq:truedata}
\yij  \;=\; \langle \xij ,  U^*\ovi \rangle + \epsij \;,\vspace*{-5pt}
\end{equation}
where $\epsij \sim \cN(0,\sigma^2)$ is the measurement noise which is independent of $\xij$. So, the optimal regressor $\owi$ for each task is given by: $\owi=\oU\ovi$. %
We denote the matrix of the optimal regressors as: $\ow = \oU (\ov)^T$ where 	$(\ov)^T = [v^{*(1)}, \ldots, v^{*(t)}]$. %
\end{assume}

\begin{assume}\label{assume:incoherence}
Let $\eigmax^*$ and $\eigmin^*$  denote the largest and smallest eigenvalues of the task diversity matrix $(r/t) (\ov)^T \ov \in {\mathbb R}^{r\times r}$ respectively. We assume that $\ov$ is $\mu$-incoherent, i.e., 
\begin{eqnarray}
	\max_{i\in[t]}  \| \ovi\|^2 \leq\; \mu\,  \eigmin^*  \;.
	\label{eq:incoherence}\vspace*{-5pt}
\end{eqnarray}\vspace*{-5pt}
\end{assume}
Our goal is to recover the subspace $U^*$, up to a nearly optimal error, from a small number of samples per task. Recovering $\oU$ enables the estimation of the regressor of any new task in the same subspace, using only $\approx O(r)$ samples. 
To this end, we minimize the empirical risk of parameter matrices $U\in{\mathbb R}^{d\times r}$  and $\v=[v^{(1)}, \ldots, v^{(t)}]^T \in{\mathbb R}^{t\times r}$: \vspace*{-5pt}
\begin{align}
& \cL(U, \v) \;=\; \sum_{i=1}^{t} \sum_{j=1}^{m} \frac12 \left(\yij - \big\langle  U\vi , \xij \big\rangle \right)^2 \;. \vspace*{-5pt}
\label{eq:risk}
\end{align}
The problem is non-convex due to bi-linearity of $U$and $V$.
$\widetilde{O}$ and $\widetilde{\Omega}$ hide logarithmic terms in $d$ and  $r$.

	\section{Main results}
\textbf{\textit{Alternating minimization}:}
We first present our main result for a standard alternating minimization method (Algorithm~\ref{alg:altmin}) when applied to the meta-learning linear regression problem in the problem setting described in Section~\ref{sec:problem}. 
\begin{thm}\label{thm:altmin_logt_informal}
	Let  there be $t$ linear regression tasks, each with $m$ samples satisfying Assumptions \ref{assume:linear_meta_problem} and \ref{assume:incoherence}, and
	\begin{align*}
		&m \geq \widetilde\Omega ( (1+ r(\SNRnl )^2) r \log t + r^2  ), 
		\; \text{ and } 
		\;
		m \t \geq \widetilde{\Omega} ( (1+(\SNRnl )^2 )({{\eigmax^*}/{\eigmin^*}}) \mu d r^2 ). 
	\end{align*}
	Then \altmin (Algorithm~\ref{alg:altmin}), 
	initialized at $\iU$ s.t. $\|(\Id-\oU(\oU)^\top ) \iU\|_F \leq \min (3/4, \Ord (\sqrt{ {\eigmin^*}/{\eigmax^* }} ) ) $ and
	run for $K=\lceil \log_2({{\eigmin^*} {\eigmin^*} {m\,t} / {\eigmax^*} {\sigma^2} {\mu\,d\,r^2}}) \rceil$ iterations, outputs $U$ so that the following holds (w.p. $\geq 1- K/(dr)^{10}$): 
	\begin{eqnarray}	
	\frac{\|(\Id - \oU (\oU)^\top)U \|_F}{\sqrt{r}} \; \leq\;  \tOrd \left( \bigg( \SNRn \bigg) \sqrt{\frac{\mu\,r\, d }{m\, t}}  \,\right) \;.
	\end{eqnarray}
\end{thm}

\noindent
{\bf Remark 1:} Our error rate is nearly optimal, as it matches best possible rate when $\ov$ is specified a priori. This is made formal in the following lower bound, which follows from \citep[Theorem 5]{tripuraneni2020provable}. The upper and lower bounds match up to polynomial factors in the incoherence $\mu$ and the condition number ${\eigmax^*}/{\eigmin^*}$.
\begin{coro}\citep[Theorem 5]{tripuraneni2020provable}  
	\label{coro:meta_learn_basis_lb_informal}    
	Let $r\leq d/2$ and $mt\geq r(d-r)$, then for all $V^*$, w.p.~$\geq 1/2${
	\small
	\begin{eqnarray*}
		\inf_{\widehat{U}} \sup_{U\in{\rm Gr}_{r,d} }  \frac{\| ( {\mathbf I}-U^*(U^*)^\top ) \widehat{U} \|_F}{\sqrt{r}} \;\geq\; \Omega\Big(  \Big(\frac{\eigmin^*}{\eigmax^*} \frac{\sigma}{\sqrt{\eigmin^*} }\Big) \sqrt{ \frac{d\,r}{m\,t}} \Big) \;,
	\end{eqnarray*}
}where $G_{r,d}$ is the Grassmannian manifold of $r$-dimensional subspaces in ${\mathbb R}^d$, the infimum for $\widehat{U}$ is taken over the set of all measurable functions that takes $mt$ samples in total from the model in Section~\ref{sec:problem} satisfying Assumption~\ref{assume:linear_meta_problem} and \ref{assume:incoherence}. 
\end{coro}	

\noindent
{\bf Remark 2}: 
To the best of our knowledge, Theorem~\ref{thm:altmin_logt_informal} presents the first efficient method for achieving optimal error rate in $\sigma$, $d$ and $r$. 
\citet{tripuraneni2020provable} propose two approaches. %
The first one is the Burer-Monteiro factorization approach, which  
achieves  a rescaled Frobenius norm error bound of $O((\sigma/\sqrt{\eigmin^*}) \sqrt{\max\{t,d\} r^2 \log(mt)/(mt)})$ if the sample size is 
$mt \geq O( \max\{t,d\} r^4 (\eigmax^*/\eigmin^*)^4  {\rm polylog}(mt,d))$ and incoherence is $\mu \leq O( \eigmax^*/\eigmin^* )$.
Several remarks on its sub-optimality are in order: $(a)$ when $t\geq d$ and for $m =  \Theta(\log(t))$, the error does not decrease  as we increase the number of tasks $t$, 
$(b)$ even when $t<d$ the error rate is sub-optimal by a factor of $\sqrt{r}$, and 
$(c)$ each task requires $m\geq O(r^4 {\rm polylog}(mt))$ samples. 
In contrast, error of \altmin decays at a rate of $1/\sqrt{t}$ when $m=\Theta(\log(t))$, 
and this rate is optimal as it matches a lower bound, and each task requires only $m=\Omega(r^2 + r\log(t))$ samples.

The second approach, based on the method-of-moments, achieves a rescaled Frobenius norm  error bound of 
$\widetilde{O}\big(  (\sigma/ \sqrt{\eigmin^*})^2  \sqrt{(\eigmax^*/\eigmin^*)(dr^2/(mt))}    +  \sqrt{\mu d r^2 (\eigmax^*/\eigmin^*)   / (mt)}   \big)$ 
if  $m\geq 2$. 
The first term  is suboptimal by a factor fo $\sqrt{r}$. 
The second term is more problematic as it does not depend on the noise $\sigma$; 
as we decrease $\sigma$, the error does not vanish even if we have enough samples to learn the parameters exactly.  
This is illustrated in the simulation result in Figure~\ref{fig:alt_mom_vs_sigma}.

\noindent{\bf Remark 3}: 
One can study the problem in a stochastic setting where we sample a task $i$ and compute stochastic gradient update for $U$ based only on $i$-th task's samples. In this case, our proof techniques could be combined with that of \citet{jainparallel2018} to obtain a nearly optimal and efficient one-pass algorithm. 
But we leave further investigation into such result for future work. 

\noindent{\bf Remark 4}: Our result holds if the initial point $\iU$ is reasonably accurate. 
One choice of initialization is to use the Method-of-Moments (MoM)~\citep{tripuraneni2020provable}. Due to sub-optimality of MoM approach (Theorem~\ref{thm:svd_mom} in Appendix), %
we get an additional sample complexity requirement of 
$mt \geq \widetilde\Omega ( (\eigmax^*/\eigmin^*) dr^2 \, ( \mu (\eigmax^*/\eigmin^*)+r (\sigma/\sqrt{\eigmin^*})^4 )$.  
Note that this does not degrade the error rate, $O(\sqrt{dr/mt})$. 

\noindent{\bf Remark 5}:
Suppose we run Algorithm~\ref{alg:altmin}, under the conditions of Theorem~\ref{thm:altmin_logt_informal} to get an estimated subspace $U$. Let a new task, whose task specific regressor $v^{*+}$ lie in $\oU$, be introduced with $m^+$ samples. Now, we can apply the step \ref{algo_line:altmin_v_update} of Algorithm~\ref{alg:altmin}, with $U$ and the new samples, to meta-learn an estimate $v^{+}$ of $v^{*+}$. Then by \citet[Theorem 4]{tripuraneni2020provable}, the mean-squared-error (MSE) of the estimated regressor is $\tOrd((\SNRnl) (\mu dr^2/mt + r/m^{+}))$. Therefore, as long as $mt$ was large enough, we only need $m^+ = \Omega(r)$ additional samples to get an arbitrarily small MSE, as opposed to $m^+ = \Omega(d)$ of trivial baseline. We also improve upon other baselines from \citep{tripuraneni2020provable} (see Table~\ref{tab:results_altmin_meta}) in terms of dependence on $\sigma$ and $t$.
\\

\noindent\textbf{\textit{Task subset selection}:}
The downside of our Algorithm~\ref{alg:altmin} is that the requirement on $m$ increase with $t$ (i.e.,~$m=\Omega(\log t)$), which is not natural as the number of required samples per task should not increase as the number of tasks increase. To remove this dependency, we propose a new algorithm (Algorithm~\ref{alg:altselect}) that samples a set of tasks at each iteration to ensure we use only the ``well-behaved" tasks.

\begin{thm}\label{thm:altmin_subset_informal}
	Let  there be $t$ linear regression tasks, each with $m$ samples satisfying Assumptions \ref{assume:linear_meta_problem} and \ref{assume:incoherence}, and
	\begin{align*}
	&m \geq \widetilde\Omega ( (\SNRnl )^2 r^2 \log t + r^2  + \log(\mu)), \; \text{ } t \geq \widetilde\Omega ( \mu^2\,r^2 ), 
	\; \text{ and } \;
	m \t \geq \widetilde{\Omega} ( (1+ (\SNRnl )^2)( {\eigmax^*}/{\eigmin^*}  ) \mu dr^2 ). 
	\end{align*}
	Then \altmins (Algorithm~\ref{alg:altselect}), 
	initialized at $\iU$ s.t. $\|(\Id-\oU(\oU)^\top ) \iU\|_F \leq \min (3/4, \Ord (\sqrt{ {\eigmin^*}/{\eigmax^* }} ),  \Ord (1/\log t ) ) $ and
	run for $K=\lceil \log_2({{\eigmin^*} {\eigmin^*} {m\,t} / {\eigmax^*} {\sigma^2} {\mu\,d\,r^2}}) \rceil$ iterations, outputs $U$ so that, w.p. $\geq 1- K/(dr)^{10}$
	\begin{eqnarray}	
	\frac{\|(\Id - \oU (\oU)^\top)U \|_F}{\sqrt{r}} \; \leq\;  \tOrd \left( \bigg( \SNRn \bigg) \sqrt{\frac{\mu\,r\, d }{m\, t}}  \,\right) \;.
	\end{eqnarray}
\end{thm}
\noindent\textbf{Remark 6}: Note that when $\SNRnl \leq 1/{\log^2 t}$, \altmins only needs $m \geq \Omega(r^2 + \log(\mu))$ samples per task. Since, \altmins selects a fraction of tasks to perform updates, the time-complexity of the method is similar to that of MLLAM. 
\begin{coro}\label{coro:altmin_subset_informal}
	Consider $t$ linear regression tasks, each with $m$ samples satisfying Assumptions \ref{assume:linear_meta_problem} and \ref{assume:incoherence} with $\sigma=0$, and
	\begin{align*}
	&m \geq \widetilde\Omega (r^2 + \log(\mu)), \;
	t \geq \widetilde\Omega ( (\mu r)^2 ), 
	\; \text{ and } \;
	m \t \geq \widetilde{\Omega} (({{\eigmax^*}/{\eigmin^*}}) \mu d r^2). 
	\end{align*}
	Then \altmins (Algorithm~\ref{alg:altselect}), initialized at $\iU$ s.t. $\|(\Id-\oU(\oU)^\top ) \iU\|_F \leq\min (3/4, \Ord (\sqrt{ {\eigmin^*}/{\eigmax^* }} ),  \Ord (1/\log t ) )$, and run for  $K$ iterations outputs $U$ so that the following holds (w.p. $\geq 1- K/(dr)^{10}$): 
	\begin{eqnarray}	
	\frac{\|(\Id - \oU (\oU)^\top)U \|_F}{\sqrt{r}} \; \leq\;  \tOrd \left( \frac{\sqrt{{\eigmin^*}/{\eigmin^*}}}{\sqrt{r} 2^{K}} \,\right) \;.
	\end{eqnarray}
\end{coro}

\noindent\textbf{Remark 7}: 
In the above corollary, for the noiseless setting $\sigma=0$, the number of samples per task does not grow with $t$, and thus it is nearly optimal. Note that the desired initialization point can be obtained using MoM.
We leave the extension to noisy setting for future work.

Proofs of Theorems~\ref{thm:altmin_logt_informal} \& \ref{thm:altmin_subset_informal} are in the Appendices~\ref{sec:altmin_logt_pf} \& \ref{sec:altmin_subset_pf}.

	\section{Alternating minimization}
\label{sec:alg}
In this section we discuss the alternating minimization algorithm we study in this paper. The algorithm follows the standard alternating minimization procedure \cite{jain2013low,csiszar} where we update the representation matrix $U$ and regressors $V$ alternately. Note that, given $U$, we can estimate each of the regressor $\vi$ {\em separately} using standard least squares regression, i.e., 
\begin{align*}
\vi=\arg\min_v \sum_{j} (\yij-\langle\xij, U v\rangle)^2\,.
\end{align*}
Similarly, given the updated regressors $\vi$'s, we can now update $U$ as: 
\begin{align*}
\hU=\arg\min_{\hU} \sum_{i,j}(\yij-\langle\xij, \hU \vi\rangle)^2\,.
\end{align*}
To ensure certain normalization, we analyze a modification of the algorithm where the next iterate for $U$ is the orthonormal subspace containing $\hU$, which we can obtain using the QR-decomposition of $\hU$. 

Our analysis requires that when we update $V$ using current $U$, we require $U$ to be independent from the training datapoints. Similarly, during the update for $U$, we require $V$ to be independent of the datapoints. We ensure the independence using two strategies: a) similar to standard online meta-learning settings \cite{finn2017model}, we select random (previously unseen) tasks and update $U$ and $V$, b) within each task, we divide the datapoints into two sets to update $V$ and $U$ separately. 

Our update for $\vi$ require $O(mr^2+r^3)$ time complexity, which can be brought down to $O( m\cdot r)$ by using gradient descent for solving the least squares. Our analysis shows that under the sample complexity assumptions of Theorem~\ref{thm:altmin_logt_informal}, each of the least squares problem has a constant condition number. So, the total number of iterations scale as $\log \frac{1}{\epsilon}$ to achieve $\epsilon$ error. If we set $\epsilon=1/poly(t,\sigma)$, then using standard error analysis, we should be able to obtain the optimal error rate in Theorem~\ref{thm:altmin_logt}. Similarly, {\em exact} update for $U$ requires $O((dr)^3+mt \cdot (dr)^2)$ time, that decreases to $O(mt\cdot d\cdot r)$ by using gradient descent updates. %

\begin{algorithm}[t!]
	\SetAlgoLined
	\DontPrintSemicolon
	\SetKwProg{myproc}{Procedure}{}{}
	{\bf Required}: Data: $\{(\xij \in \bR^d, \yij \in \bR)\}_{j=1}^m$ for all $1\leq i\leq t$, $K$: number of steps.\\
	\nl Initialize $U \leftarrow \iU$ \\
	\nl Randomly shuffle the tasks \{1,\ldots,t\}\\
	\For{$1\leq k\leq K$}{ 
		\nl $ {\cal T}_{k} \gets [1 + \frac{t(k-1)}{K},  \frac{tk}{K} ]   $\\
		\For{$  i \in {\cal T}_{k}  $}{
			\nl $\displaystyle \vi \gets \arg\min_{\widehat{v} \in {\mathbb R}^{r}}   \sum_{j\in[m/2]}  \left(\yij - \big\langle  U \widehat{v} , \xij \big\rangle \right)^2 $ \label{algo_line:altmin_v_update}\\
		}
		\nl $ \displaystyle \hU \leftarrow  \arg\min_{\widehat{U}\in {\mathbb R}^{d\times r}}  \sum_{i={\cal T}_k } \sum_{j=1+\frac{m}{2} }^{m}  \left(\yij - \big\langle  \widehat{U} \vi , \xij \big\rangle \right)^2 $\\
		\nl $ U \gets \mathrm{QR}(\hU)$
	}
	\Return U
	\caption{MLLAM: Meta-Learning Linear regressors via Alternating Minimization}
	\label{alg:altmin}
\end{algorithm}

\subsection{Subset Selection}
Algorithm~\ref{alg:altmin} computes regressors $\vi$ for each of the task and use that to update $U$. Now, the Hessian for $\vi$ is given by: $H^{(i)}=\frac{1}{m}U^\top \sum_j \xij(\xij)^\top U$. For sub-Gaussian $\xij$'s, $\|H^{(i)}-I\|\leq \sqrt{{r}/{m}}\sqrt{\log 1/\delta}$ with probability $1-\delta$. This implies that if, $m$ is independent of $t$, and if $t\gg m$ then the Hessian of some of the tasks can be highly ill-conditioned, leading to large estimation error in some of the regressors, which in turn leads to a large error in estimation of $U$. In Theorem~\ref{thm:altmin_logt_informal}, we avoid this issue by selecting $m$ such that it grows logarithmically with $t$. 

However, intuitively the number of required samples for each task should not increase with the number of tasks, especially in noise-less settings, where $t\geq \widetilde{\Omega}(d)$ should be enough to ensure exact recovery of $U$. Practically, also a few poor tasks should not affect representation of the data significantly. So, in Algorithm~\ref{alg:altselect}, we propose a method to ignore the poor ill-conditioned tasks. To ensure this, we compute Hessian  $H^{(i)}$ for each task, and ignore tasks whose Hessian's eigenvalue is small (see Line~\ref{algo_line:altmin_subset_select} in Algorithm~\ref{alg:altselect}). As mentioned in Theorem~\ref{thm:altmin_subset_informal}, while we condition on a task being {\em good}, we are still able to provide a similar result as Theorem~\ref{thm:altmin_logt_informal} but with $m$ which is independent of $t$ in low-noise settings, e.g., when $\sigma\leq 1/\log t$. 
\begin{algorithm}[t!]
	\SetAlgoLined
	\setcounter{AlgoLine}{0}
	\DontPrintSemicolon
	\SetKwProg{myproc}{Procedure}{}{}
	{\bf Required}: Data: $\{(\xij \in \bR^d, \yij \in \bR)\}_{j=1}^m$ for all $1\leq i\leq t$, $K$: number of steps.\\
	\nl Initialize $U\gets \iU$ \\
	\nl Randomly shuffle the tasks \{1,\ldots,t\}\\
	\For{$ 1 \leq i \leq t$}{
		\nl $\Si\gets \frac{2}{m} \sum_{j\in[m/2]} \xij (\xij)^\top$
	}
	\For{$1\leq k\leq K$}{
		\nl $ {\cal T}_{k} \gets \big\{  i\in [1 + \frac{t(k-1)}{K},  \frac{tk}{K} ] \;\big|\;\; \sigma_{\rm max}(U^\top \Si U)\leq 10;\ \sigma_{\rm min}(U^\top \Si U)  \geq \frac12  \big\}$ \label{algo_line:altmin_subset_select}\\ \\
		\For{$ i \in {\cal T}_k$}{
			\nl $\displaystyle \vi \gets \arg\min_{\widehat{v} \in {\mathbb R}^{r}}   \sum_{j\in[m/2]}  \left(\yij - \big\langle  U \widehat{v} , \xij \big\rangle \right)^2 $\\
		}
		\nl $ \displaystyle \hU \leftarrow  \arg\min_{\widehat{U}\in {\mathbb R}^{d\times r}}  \sum_{i={\cal T}_k } \sum_{j=1+\frac{m}{2} }^{m}  \left(\yij - \big\langle  \widehat{U} \vi , \xij \big\rangle \right)^2 $\\
		\nl $ U \gets \mathrm{QR}(\hU)$
	}
	\Return U
	\caption{MLLAMS: Meta-Learning Linear regressors via Alternating Minimization over task Subsets}
	\label{alg:altselect}
\end{algorithm}

\section{Proof sketch for noiseless case} 
\label{sec:sketch}

Here we provide proof sketches of Theorem~\ref{thm:altmin_logt_informal}. 
To highlight the main ideas behind our analysis, we start with 
the simplest case when there is no noise ($\sigma^2=0$) and 
all the regressors lie on a single dimensional subspace ($r=1$). 
The analysis gets quite challenging as we go to multi-dimensional shared subspace ($r>1$), 
and we illustrate these challenges and how to resolve them in Section~\ref{sec:sketch_r}.

\subsection{Proof sketch for the one-dimensional case}
Let $\ou \in \bR^d$ be the unit vector of the one-dimensional  true subspace, and $v^* \in \bR^t$ the vector of the true regressor coefficients of the $t$ tasks. In the noiseless setting ($\epsij = 0$), the $k$-th step of \altmin 
can be written as follows. 
{\small
\begin{align*}
&\text{For all $i\in \cT_k$ } \nonumber \\
&\;\;\;\;\; \vi \leftarrow  ({u^\top \Si_{1} u })^{-1} {u^\top \Si_{1} (\ou)} \ovi \;,  \\
&\hu \leftarrow   \Big( \sum_{i \in \cT_k} (\vi)^2 \Si_{2} \Big)^\inv\Big( \,   \sum_{i\in \cT_k}  \ovi \vi \Si_{2} \ou \,\Big) \;,\ \pu  \leftarrow  \frac{\hu}{\|\hu\|},%
\end{align*}
}where $\Si_{\ell} = \frac{2}{m} \sum_{j=(\ell-1)m/2 + 1}^{\ell m/2} \xij (\xij)^\top$ is the data covariance matrix of a half of the dataset $[m]$ of task $i \in [t]$. Our incoherence condition for rank-$1$ case simplifies to $\|v\|_\infty^2 \leq \frac{\mu}{t} \|v\|^2$.
The distance between two unit norm vectors $\u$ and $\ou$ is commonly measured 
by the angular distance defined as $\sin\theta(\u, \ou) \triangleq \|(\Id - \ou (\ou)^\top) u\|^{1/2}$, where $\Id - \ou (\ou)^\top$ is the projection operator to the sub-space orthogonal to $\ou$. 
In the following we let  $q \triangleq \Ip{\ou}{\u}$ and use the relation 
 $\sin\theta(u, \ou) = \|u - \ou q\|$ in the analysis.  
 We use the fact that   if we have a good previous iterate $u$ close to $\ou$, i.e.~$\sin\theta(\u, \ou) \leq 3/4$, then $1/2 \leq |q| \leq 1$. 

Our analysis shows that 
we get geometrically closer to the true subspace $\ou$ at every iteration in this $\sin\theta$ distance, 
when initialized sufficiently close to $\ou$. 

Our strategy is to show that the $v$-update achieves $|\vi - q^{-1} \ovi| \leq C  \|\ovi\|\sin\theta(\u, \ou) $ for some constant $C$, and the $u$-update achieves  $\sin\theta(\pu, \ou) \leq   (c/\|v^*\|) \|v - q^{-1} v^*\| )$  where the constant $c$ can be made as small as we want 
in the assumed sample regime. Together, they imply the desired theorem. 
\\

\noindent\textbf{$v$-update:} 
We can write $\vi q^{-1} - \ovi$ as
\begin{align*}
&\vi - q^{-1} \ovi %
= {u^\top \Si_{1} (q \ou - u)}{(u^\top \Si_{1} u )^{-1}} q^{-1} \ovi \,.
\end{align*}
In expectation, $\|\bE[u^\top \Si_{1} (q \ou - \u)]\| = \|\u^\top (q \ou - \u)\| = 1-q^2 \leq (\sin\theta(u, \ou))^2$ and $\bE[\u^\top \Si_{1} \u] = \|\u\|^2= 1$. Therefore, by Lemma~\ref{lem:v_update_rkr_logt}, if $\sin\theta(u, \ou) \leq \frac1{32}$ and there is enough samples per task, i.e.~$m \geq \Omega(\log({t}/{K\,\delta}))$, we can bound their deviations in terms of $\sin \theta(u, \ou)$. 
This implies that, 
with a probability of at least $1 - \delta/2$, %
\begin{align}
\frac{|\vi - q^{-1} \ovi |}{|\ovi |} &\leq \frac{\sin\theta (u, \ou)}4  \text{\,,  for all } i \in \cT_k,
\label{eq:pf_sketch_rk1_v_update}
\end{align}
where we used the fact that $|q|\geq 1/2$. 
This in turn implies that $ (1/4) |\ovi| \leq |\vi|$ %
and $v$ is incoherent.
\\

\noindent\textbf{$u$-update:}  %
We bound the distance between $\hu$ and $\ou$:
\begin{align}
&\hu - \ou q \nonumber \\
&=  \Big( \underbrace{\sum_{i \in \cT_k } \frac{(\vi)^2}{\|v\|^2} \Si_{2} \Big)^\inv}_{:=A} \Big( \,   \underbrace{\sum_{i\in \cT_k }  \frac{\vi \dvi}{\|v\|^2} \Si_{2}}_{:= \widehat{H}} \ou q  \,\Big), 
\end{align}
where $\dvi = q^{-1} \ovi - \vi$.
Notice that, in expectation, $\bE[A] = \Id$ and $\bE[\widehat{H} \ou q] = \frac{v^\top h}{\|v\|^2} \ou q \leq \frac{\|h\|}{\|v\|}$. Therefore, by Lemma~\ref{lem:u_update_rkr_logt}, when there are enough samples, i.e.~${mt} \geq K\Omega(\mu d \log(\frac1\delta))$ deviations form these expected values can be bounded using the distance between $v$ and $v^*$, $\|h\|$.
That is with a probability of at least $1 - \frac\delta2$, $A$ is invertible and well-conditioned, 
$$A^\invert= \Id + E_1 \text{, \;\;and\;\; } H \ou q =  \frac{v^\top h}{\|v\|^2} \ou q + e_2,$$ where $\|E_1\| \leq \frac1{16}$ and $\|e_2\| \leq \frac1{32} \Big( \frac{\|h\|}{\|v\|} + \sqrt{\frac{t}{\mu}} \frac{\|h\|_\infty}{\|v\|} \Big)$. Note that we had to critically use incoherence of intermediate $v$ to bound $e_2$.
Therefore
\begin{align*}
\hu - \ou q = \underbrace{\frac{v^\top h}{\|v\|^2} \ou q}_{:= \hu_{\parallel}} + \underbrace{q \frac{v^\top h}{\|v\|^2} E_1 \ou + (\Id+E_1) e_2}_{:=f}\,.
\end{align*}
Notice that $\hu_{\parallel}$ is parallel to $\ou$. Rest of the terms are grouped together as $f$. The angle distance $\sin(\pu, \ou)$ only depends on the portion of $\pu$ which lie in the orthogonal subspace to $\ou$. Therefore, $\|\hu_{\parallel}\|$ does not directly contribute to the distance, and this is formalized below. Clearly, $\|(\Id - \ou (\ou)^\top ) \pu \| = \min_{q^+} \|\pu - \ou q^+\| $. This follows from the trivial solution of the scalar quadratic problem $\min_{q^+ \in \bR} \|u - \ou q^+\|^2$. Thus,
\begin{align}
\sin\theta(\pu, \ou) %
&= \min_{q+} \|\pu - \ou q^+\| \nonumber \\
&\leq \Big\|\frac{\hu}{\|\hu\|} - \Big(1+ \frac{h^\top v}{\|v\|^2}\Big)\ou \frac{q}{\|\hu\|} \Big\| \nonumber \\
&\leq \frac{\|f\|}{ \|\hu\|} \leq \frac{\|f\|}{q\|\ou\| - \|f\| - {\|h\|}/{\|v\|}}\,. \label{eq:pf_sketch_rk1_u_update}
\end{align}
\\

\noindent{\bf Putting them together:}
We bound $f$ using definitions of $E_1$ and $e_2$, incoherence, and \eqref{eq:pf_sketch_rk1_v_update} as
\begin{align*}
\|f\| &\leq \frac1{16} \frac{\|h\|}{\|v\|} + \frac1{32} \Big( \frac{\|h\|}{\|v\|} + \sqrt{\frac{t}{\mu}} \frac{\|h\|_\infty}{\|v\|} \Big) \leq \frac18 \sin\theta (u, \ou)\,.
\end{align*}
Combining this with \eqref{eq:pf_sketch_rk1_u_update}, we see that with a probability of at least $1 - \delta$, the angle distance geometrically decreases at each step, i.e.~
\begin{align}
\sin\theta(\pu, \ou) &\leq \frac12 \sin \theta (u, \ou).
\end{align}
Finally, if the initialization is good, i.e.~$\sin\theta(u_{\mathrm{init}}, \ou)  \leq \frac1{16}$, we can unroll the above inequality across iterations. Taking union bound over the iterations we get that, with a probability of at least $1 - K \delta$, the output $u$ after $K$ iterations satisfies
\begin{align}
\sin\theta(u, \ou) &\leq \frac1{2^K} \sin\theta(u_{\mathrm{init}}, \ou).
\end{align}
To achieve this, we need at least $m \geq \Omega(\log(\frac{t}{K\delta}))$ samples per task and at least $mt \geq \Omega(K \mu d \log(\frac1\delta))$ total samples.

\subsection{Proof sketch for the $r$-dimensional case}
\label{sec:sketch_r} 
Here we do not use $\sin \theta_1(U, \ou)$ distance, as the analysis of $\sin \theta_1$ gets more complicated in the general $r$-dimensional case. Therefore we use $\ell$-$2$ norm based error, $\Delta(U, \oU) := (\sum_{r'=1}^r \sin^2 \theta_{r'}(U, \oU))^{1/2} := \|(\Id - \oU (\oU)^\top) U\|_F$. Let $Q = (\oU)^\top U$, then $\Delta(U, \oU) = \|U - \oU Q\|_F$, and $1/2 \leq \|Q\| \leq 1$ if $\Delta(U, \oU) \leq 3/4$.
\begin{align*}
&\text{For all $i\in \cT_k$ } \nonumber \\
&\;\;\;\;\; \vi \leftarrow  (U^\top \Si_{1} U)^\inv U^\top \Si_{1} \oU \ovi \;, \\
&\hU \leftarrow   \Big(\cA^\inv\Big( \sum_{i\in \cT_k}  \Si_{2} \oU \ovi (\vi W^{-\frac12})^\top  \,\Big)\Big) W^{-\frac12} \;, \\
&U  \leftarrow  \mathrm{QR}(\hU)\;,
\end{align*}
where $W = V^\top V$, $\cA:\bR^{d \times r} \to \bR^{d \times r}$ is linear operator such that $\cA(U) = \sum_{i \in \cT_k} \Si_{2} U W^{-\frac12}\vi (\vi)^\top W^{-\frac12}$, and $\Si_{\ell}$ are defined as in the one-dimensional case.
\\

\noindent\textbf{$V$-update:} We will prove that $\|\vi - Q^{-1} \ovi\| = \Ord(\Delta(U, \oU))$. Let $\dvi := \vi Q^{-1} - \ovi$, then
\begin{align*}
\dvi 
= (U^\top \Si_{1} U )^\inv \underbrace{U^\top \Si_{1} (\oU Q - U) Q^\inv \ovi}_{:=G}.
\end{align*}
Notice that, in expectation, $\|\bE[U^\top \Si_{1} U]\| = 1$ and $\|\bE[G]\| = \|U^\top  (\oU Q - U)\| = \|Q^\top Q - \Id \| = \Delta^2(U, \oU)$. Therefore, by Lemma~\ref{lem:v_update_rkr_logt}, if $\Delta^2(U, \oU) \leq \frac1{32}$ and there is enough samples per task, i.e.~$m \geq \Omega(r\,\log(\frac{t}{K\,\delta}))$, we can bound their deviations in terms of $\sin \theta(u, \ou)$. 
This implies that, 
with a probability of at least $1 - \delta/2$,
\begin{align}
\|\dvi \| &\leq \frac{\|\ovi \| \Delta^2(U, \oU)}4 \text{\,, for all } i \in \cT_k.
\label{eq:pf_sketch_rkr_v_update}
\end{align}
Furthermore, $\|\vi\| \leq 4\|\ovi\|$ and $V$ is incoherent.
\\

\noindent\textbf{$U$-update:} 
We bound the distance between $\hU$ and $\oU$:
\begin{align*}
(\hU - \oU Q) W^{\frac12}
&=  \cA^\inv \Big( \,   \underbrace{\sum_{i\in \cT_k } \Si_{2} \oU Q \dvi (\vi)^\top W^{-\frac12} }_{:= -\hcH(\oU Q)}  \,\Big).
\end{align*}
Notice that, in expectation, $\bE[\hcH (\oU Q)] = \cH (\oU Q) := \oU Q \sum_{i\in \cT_k } \dvi (\vi)^\top W^{-\frac12}$ and $\cH (\oU Q) \leq \|H\|_F$ and $\bE[\cA]$ is the identity map $\cI$. Like in the $1$-dimensional case, by Lemma~\ref{lem:u_update_rkr_logt}, when there are enough samples, i.e.~${mt} \geq K\Omega(\mu d r^2 \log(\frac1\delta))$ deviations from these expected values can be bounded using the distance between $\v$ and $\ov$, $\|H\|$. %
That is, with a probability of at least $1 - \delta/2$, $\cA$ is invertible and well-conditioned in Frobenius operator norm, 
$$\cA^\invert= \cI + \cE_1 \text{, \;\;and\;\; } \hcH(\oU Q) =  \cH(\oU Q) - E_2,$$ where $\|\cE_1\|_F \leq 1/{16}$ and $\|E_2\|_F \leq 1/{32} ({\|H\|_F}+ \sqrt{{t}/{\mu}} {\|H\|_\inftyone} )$. Note that we had to critically use incoherence of intermediate $V$ to bound $E_2$.
Therefore,{\small
\begin{align*}
(\hU - \oU Q)W^{\frac12} = -\cH(\oU Q) - \underbrace{cE_1  \cH(\oU Q) + (\cI+\cE_1) E_2}_{:=F}\,.
\end{align*}}
Now, using similar arguments as in the one-dimensional case, we get
\begin{align*}
\Delta(&\pU, \oU) %
\leq \Big\|\hU R^{-1} - \oU Q + \cH(\oU Q) \Big\|_F \|W^{-\frac12}\| \nonumber \\
&\leq \frac{\|F\|_F}{\|R^{-1}\|} \leq \frac{\|F\|_F \lambda_r^{-\frac12}}{\|Q \oU\| - (\|F\|_F + \|H \|_F) \lambda_r^{-\frac12}}\,. \label{eq:pf_sketch_rkr_u_update}
\end{align*}
\\

\noindent{\bf Putting them together:}
Using similar arguments as in one-dimensional case, 
if the initialization is good, i.e.~$\Delta (\iU, \oU)  \leq 1/{16}$, 
we can show that
with a probability of at least $1 - \delta$, the next iterate $\pU$ satisfies: 
$\Delta (\pU, \oU) \leq \frac12 \Delta (U, \oU)\,.$ 
To achieve this, we need at least $\Omega(r\log(\frac{t}{K\delta}))$ samples per task ($m$) and at least $\Omega(K \mu d r^2 \log(\frac1\delta))$ total samples ($mt$). Result now follows by applying the above result $K$ times.

	\begin{figure*}[t!]
	\centering
	\begin{subfigure}[t]{0.3\linewidth}
		\centering
		\includegraphics[width=0.9\textwidth]{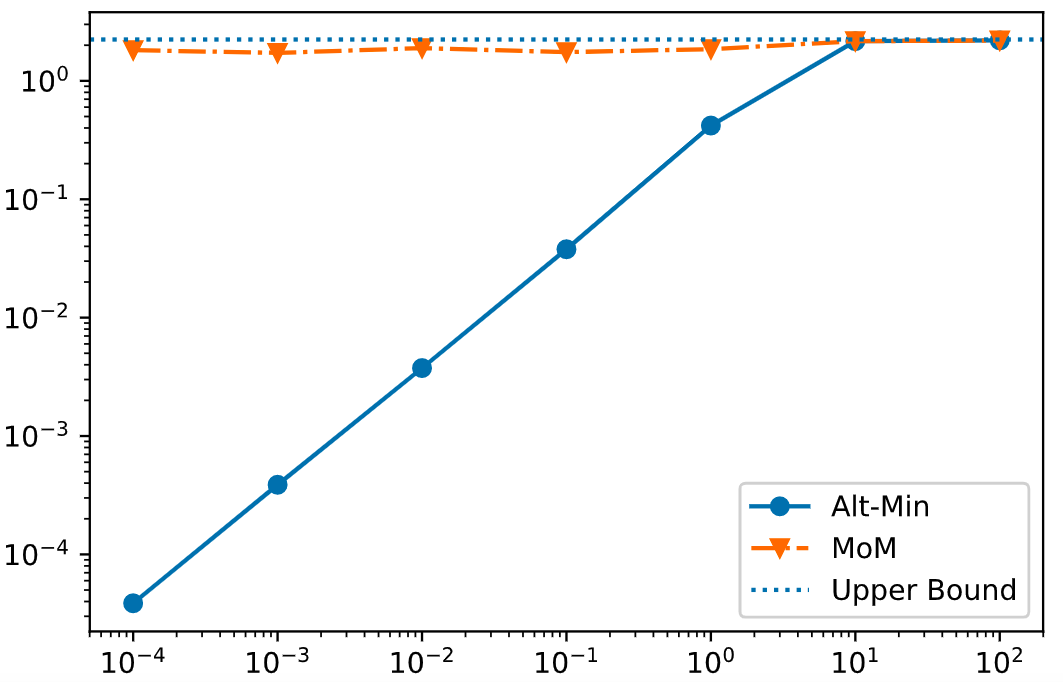}
		\put(-95,-6){\tiny {Noise magnitude {($\sigma$)}}}
		\put(-140, 10){\scalebox{.6}{\rotatebox{90}{{Error: $\|(\Id - \oU (\oU)^\top ) U\|_F$}}}}\vspace*{-5pt}
		\caption{}
		\label{fig:alt_mom_vs_sigma}
	\end{subfigure}\hspace*{10pt}
	\begin{subfigure}[t]{0.3\linewidth}
		\centering
		\includegraphics[width=0.9\textwidth]{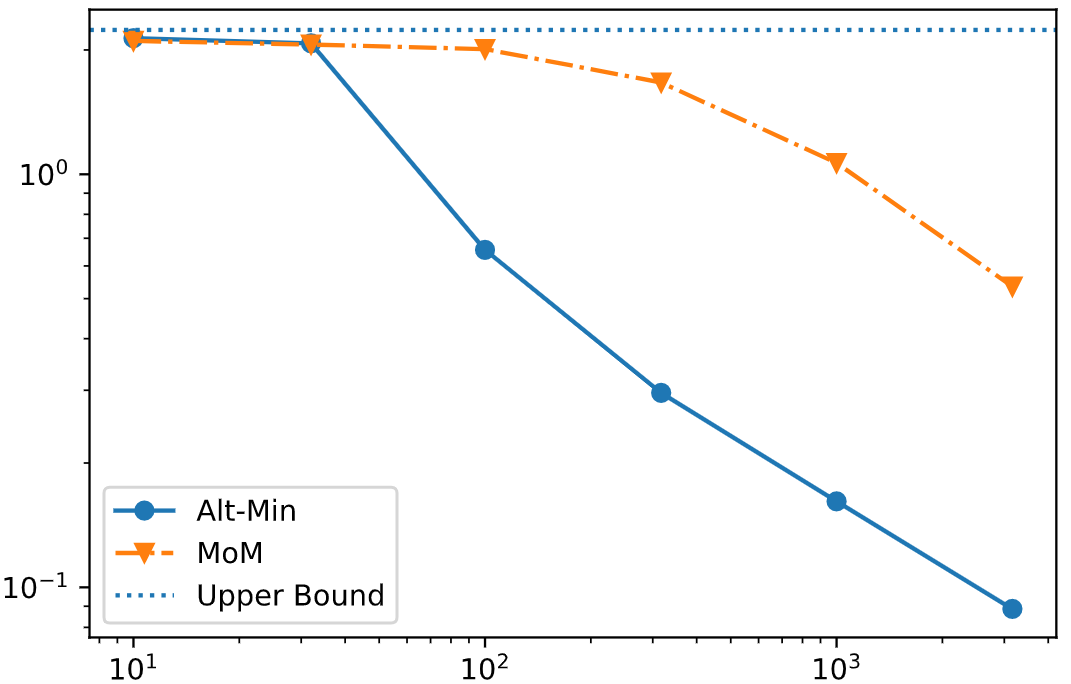}
		\put(-90,-6){\tiny {Number of tasks {($t$)}}}
		\put(-140, 10){\scalebox{.6}{\rotatebox{90}{{Error: $\|(\Id - \oU (\oU)^\top ) U\|_F$}}}}\vspace*{-5pt}
		\caption{}
		\label{fig:alt_mom_vs_t}
	\end{subfigure}%
	\hspace*{10pt}
	\begin{subfigure}[t]{0.3\linewidth}
		\centering
		\includegraphics[width=0.9\textwidth]{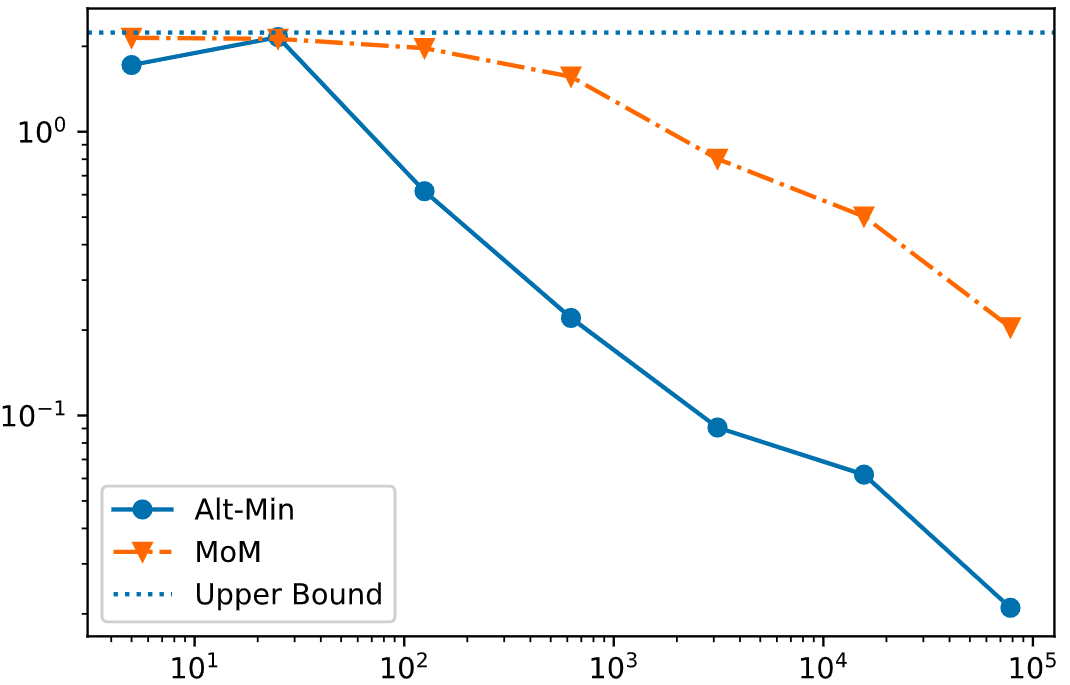}
		\put(-100,-6){\tiny {Number of samples per tasks {($m$)}}}
		\put(-140, 10){\scalebox{.6}{\rotatebox{90}{{Error: $\|(\Id - \oU (\oU)^\top ) U\|_F$}}}}\vspace*{-5pt}
		\caption{}
		\label{fig:alt_mom_vs_m}
	\end{subfigure}\vspace*{-8pt}%
	\caption{
		(a): \altmin achieves vanishing error as noise decreases, whereas  the error with Method-of-Moments stay bounded away from zero.
		(b), (c): \altmin incurs significantly smaller error in estimation of true subspace $\oU$ than MoM, both for growing number of tasks ($t$)  and for growing number of samples per task ($m$). %
	}
	\label{fig:alt_mom_vs_mt}
\end{figure*}
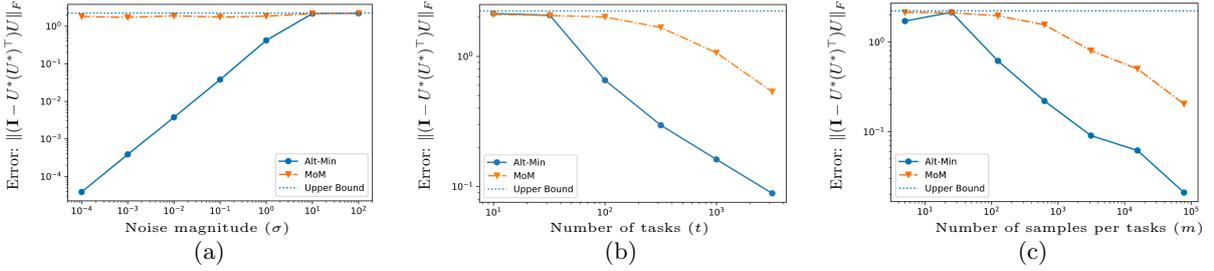

\section{Experimental results}
\label{sec:expts}

In this section we empirically compare the performance of \altmin (Alt-Min, Algorithm~\ref{alg:altmin})  against Method-of-Moments (MoM)~\cite{tripuraneni2020provable}. We generate data samples with dimension $d=100$ and generate random subspace $\oU$ of rank $r=5$. We sample the task regressor coefficients as ~$\vi \sim \cN(0, \Id)$. In all our experiments, we initialize \altmin uniformly at random and run it for $K=20$ iterations. In all the figures, the blue straight line with circular marker denotes the \altmin algorithm, the orange dashed and dotted line with inverted triangular marker denotes the MoM, and  the blue dotted line parallel to x-axis represents the theoretical upper-limit $\sqrt{r}$ of the Frobenius norm distance, $\|(\Id - \oU (\oU)^\top) U\|_F$. In all the figures we use log-scaled x and y axes.

Figure~\ref{fig:alt_mom_vs_sigma} plots subspace estimation  error ($\|(\Id - \oU (\oU)^\top) U\|_F$) against the standard deviation $\sigma$ of the regression noise, $\epsij \sim \cN(0, \sigma^2)$; see \eqref{eq:truedata}. We vary $\sigma$ from $10^{-4}$ to $10^{2}$, while fixing the number of tasks at $t=200$ and the number of samples per task at $m=25$. Clearly, as predicted by Theorem~\ref{thm:altmin_logt_informal}, our \altmin (Alt-Min) algorithm achieves a smaller error than MoM over all values of $\sigma$. Error of \altmin is linearly proportional to $\sigma$. As predicted by Theorem~\ref{thm:svd_mom} (in Appendix), the distance of MoM is a constant multiple of $\sqrt{\frac{dr^3}{mt}} = \sqrt{r}$ for all values of $\sigma$, and it does not improve when $\sigma$ decreases.

Figure~\ref{fig:alt_mom_vs_t} plots the subspace error against the number of tasks $t$. We vary $t$ from $10$ to $3163$, while the number of samples per task is fixed at $m=25$ and $\sigma=1$. In Figure~\ref{fig:alt_mom_vs_m}, we plot the the error against the number samples per tasks $m$. We vary $m$ from $5$ to $78125$, while fixing the number of tasks at $t=20$ and the standard deviation of the regression noise at $\sigma=1$. In both of these figures, we observe that, \altmin (Alt-Min) achieves much smaller subspace error than the MoM. Furthermore, as predicted by Theorems~\ref{thm:altmin_logt_informal} and~\ref{thm:svd_mom} (in Appendix), the squared error rate for both \altmin and MoM decreases linearly $m$ and $t$. %

Note that even though we randomly initialize our \altmin algorithm, it still performs better than the baseline MoM. Similar observations have been made for other non-convex algorithms for solving low-rank problems~\cite{chen2019gradient}. This suggests that the initialization requirement of Theorem~\ref{thm:altmin_logt_informal}, $\|(\Id - \oU (\oU)^\top)U\|_F \leq \Ord(\sqrt{\eigmin^*/\eigmax^*})$ may be an artifact of the analysis or may be practically insignificant.

	\section{Conclusion}
\label{sec:conclusion}
In this paper, we analyzed an alternating minimization method for the problem of linear meta-learning, a simple but canonical problem in meta-learning. We showed that Algorithm~\ref{alg:altmin} that alternately learns the shared representation matrix across tasks and the task-specific regressors, can provide nearly optimal error rate along with nearly optimal per-task and overall sample complexities. To the best of our knowledge, we provide the first result with optimal error rate --- that scales appropriately with the noise in observations --- while still ensuring per-task sample complexity to be nearly independent of $d$ (the dimensionality of data), which is a key requirement in meta-learning as individual tasks are data-starved.  We also proposed and analyzed a subset selection based method that further improves per-task sample complexity and ensures that it is {\em independent} of the number of tasks for noise-less setting. 

The work leads to several interesting future directions and questions. For the non-linear version of the problem, ensuring optimal error rate with  optimal per-task sample complexity is an interesting open question. Understanding and contrasting standard {\em MAML} techniques for the linear and non-linear problem is another exciting direction, which is already seeing a fair amount of interest \cite{fallah2020convergence,saunshi2020sample}. Finally, analyzing alternating minimization methods with stochastic gradients and streaming tasks is another promising direction. %

	\ifarxiv
	\printbibliography
	\fi

	\clearpage 
	
\onecolumn
\appendix

\section*{Appendix}

This appendix contains proofs for the claims mentioned main text. Section~\ref{sec:analysis_rkr_logt} and \ref{sec:analysis_rkr_subset} contain the analyses of Algorithm~\ref{alg:altmin} and \ref{alg:altselect}, respectively. Section~\ref{sec:past_corollaries} contains corollaries of some known results. Section~\ref{sec:technical_lemmas} contains some general technical lemmas used in this paper.

\section{Analysis of \altmin (Algorithm~\ref{alg:altmin})}
\label{sec:analysis_rkr_logt}
Initialized at $U$, the $k$-the step of alternating minimization-based \altmin (Algorithm~\ref{alg:altmin}) is:
\begin{eqnarray}
\vi &\leftarrow&  (U^\top \Si_1 U )^\inv( (U^\top \Si_1 U^* ) \ovi + U^\top \zi)\;,\;\;\;\;\;\; \text{ for }i\in \cT_k = [1 + (k-1)t/K, t k/K]  \label{eq:r_vupdate_logt}\\
\hU &\leftarrow&   \cA^\inv\Big( \,   \sum_{i\in [t]} \Si_2  U^* \ovi (\vi)^\top + \zi (\vi)^\top \,\Big) \;, \label{eq:r_uupdate_logt}\\
\pU & \leftarrow & \mathrm{QR}(\hU)\;,
\end{eqnarray}
where $\pU$ is the next iterate, $\Si_1 = \frac2m \sum_{j \in [1, m/2]} \xij (\xij)^\top$, $\Si_2 = \frac2m \sum_{j \in [1+m/2, m]} \xij (\xij)^\top$, $z^{(i)} \triangleq (1/m)\sum_{j\in[m]} \varepsilon_j^{(i)}x_j^{(i)}$ and ${\cal A}:{\mathbb R}^{d\times r} \to {\mathbb R}^{d\times r}$ is a  self-adjoint linear operator such that ${\cal A}(U)=  \sum_{i\in T} \Si  U \vi (\vi)^\top$. The self-adjointness of $\cA$ follows from the symmetry of $\Si$ when using cyclic property of trace as follows
\begin{align}
\Ip{U_2}{\cA(U_1)} = \sum_{i \in T}\Ip{U_2}{\Si U_1 \vi (\vi)^\top)} &= \sum_{i \in T} \tr(U_2^\top \Si U_1 \vi (\vi)^\top) \nonumber \\
&= \sum_{i \in T} \tr(\vi (\vi)^\top U_2^\top \Si U_1 ) = \Ip{\cA(U_2)}{U_1}
\end{align}

\noindent
{\bf Incoherence.} $ \max_i \|\ovi\|^2 \leq (\mu \,r/\t)  \eigmin(\sum_{i\in[t]} \ovi (\ovi)^\top)$, 
and we define $\nu = (1/t)\eigmin(\sum_{i\in[t]} \ovi (\ovi)^\top)$. Notice that, this non-standard definition of incoherence is related to the standard definition: $W^* = (V^*)^\top V^* = \sum_{i \in [t]} \ovi (\ovi)^\top$, $V^* = \tilde{V}^* R^*$ (QR-decomposition), $\max_i \|\widetilde{v}^{*(i)}\|^2 \leq \widetilde{\mu} \,r/\t$, as follows $\mu = \widehat{\mu}(\sigma_1^2(R^*)/\sigma_r^2(R^*))$.

\begin{thm}\label{thm:altmin_logt}
	Let  there be $t$ linear regression tasks, each with $m$ samples satisfying Assumptions \ref{assume:linear_meta_problem} and \ref{assume:incoherence},
	and $K = \lceil \log_2(\frac{({\eigmin^*}/{\eigmax^*}) mt}{\mu dr^2}) \rceil$, $\|(\Id - \oU (\oU)^\top)\iU\|_F \leq \min \Big( \frac34, \Ord\Big(\sqrt{\frac{\eigmin^*}{\eigmax^*}  \frac1{\log(t/K)} } \Big)\Big)$,
	$m \geq \Omega\Big( (1+r \Big(\SNRre\Big)^2) r \log(\frac{t}{\delta}) + r^2 \log(\frac{K}{\delta}) \Big)$, 	$t \geq \Omega(\mu^2 r^3 K \log(\frac{K}\delta))$,
	and 
	$mt \geq \Omega\Big({\mu d r^2 K \frac{\eigmax^*}{\eigmin^*}}  \Big(\log(\frac{t}\delta) + \Big(\SNRre \Big)^2 \log^2(\frac{t}{\delta}) \log(\frac{rK}\delta) \Big)\Big)$.
	Then, for any $0 < \delta < 1$, after 
	$K$
	iterations, \altmin (Algorithm~\ref{alg:altmin}) returns an orthonormal matrix $U \in \bR^{d \times r}$, such that with a probability of at least $1 - \delta 
	$
	\begin{align}
	\frac1{\sqrt{r}} \|(\Id - \oU (\oU)^\top)U \|_F 
	&\leq \Ord\Big(\SNRre \sqrt{\frac{\mu d r K \log(\frac{t}{\delta}) \log(\frac{rK}\delta)}{mt}} \Big)
	\end{align}
	and the algorithm uses an additional memory of size $\Ord(d^2 r^2)$.
\end{thm}

A proof is in Section~\ref{sec:altmin_logt_pf}.
\\

\noindent{\bf Initialization.} If we initialize \altmin (Algorithm~\ref{alg:altmin}) with Method-of-Moments (Theorem~\ref{thm:svd_mom}), we need at least
\begin{align}
mt &\geq \widetilde\Omega \Big(
\frac{\eigmax^{*2}}{\eigmin^{*2}} {\mu dr^2} + \Big( \SNRre \Big)^4 \frac{\eigmax^*}{\eigmin^*}  dr^3\Big) 
\end{align}
initial number of samples, where $\widetilde{\Omega}$ hides $\mathrm{polylog}$ factors.

\subsection{Proof of Theorem~\ref{thm:altmin_logt}}
\label{sec:altmin_logt_pf}

\textbf{Proof sketch:} 
We first prove that distance between $\oU$ and $U$ decreases at each iteration up to some additional noise terms. Then this per iterate result is unrolled to obtained the final guarantees. 

First we focus on the $k$-th iterate. In this analysis, unless specified $[t]$, represents the $k$-th $K$-way partition used for the $k$-th iterate.

In the analysis of an iterate we denote the current iterate using $U$ and the next iterate using $\pU$.
First we prove that the distance between the true $\ovi$ and the current $\vi$ is approximately upper-bounded by multiple of distance between $U$ and $\oU$. Next we prove that distance between $\pU$ and $\oU$ is approximately a fraction of the distance between  $\ovi$ and $\vi$. Finally, combining the above two results gives us desired result.
\\

\medskip\noindent\textbf{Preliminaries:}
Let $Q = (\oU)^\top U$. Using Lemma~\ref{lem:distance-relations}, if $\|U - \oU (\oU)^\top U \|_F < 1$, $Q$ is invertible. Let $Q^{-1}$ be the right inverse of $Q$, i.e.~$Q Q^{-1} = \Id$. Let $W = (V^*)^\top V^* = \sum_{i \in [t]} \ovi (\ovi)^\top$, and $\eigmax^* = \max_{\|z\| = 1} z^\top W^* z$ and $\eigmin^* = \min_{\|z\| = 1} z^\top W^* z$.

\medskip\noindent\textbf{Update on $V$:}
Let $\dvi = \vi - Q^\invert \ovi$ and $\dv^T = [h^{(1)} h^{(2)} \ldots h^{(t)}]$. Let $\|\dv\|_F \triangleq \sqrt{\sum_{i \in [t]} \|\dvi\|^2}$ and $\|\dv\|_\inftyone \triangleq \max_{i \in [t]} \|\dvi\|$. Let $W = V^\top V = \sum_{i \in [t]} \vi (\vi)^\top$, and $\eigmax = \max_{\|z\| = 1} z^\top W z$ and $\eigmin = \min_{\|z\| = 1} z^\top W z$.

\begin{lemma}\label{lem:v_update_rkr_logt}
If $\|(\Id - \oU (\oU)^\top)U\|_F \leq \min \Big( \frac34, \Ord\Big(\sqrt{\frac{\eigmin^*}{\eigmax^*} \frac1{\log(t/K)} } \Big)\Big)$ and $m \geq \Omega\Big( \Big(\SNRre \Big)^2 r^2 \log(\frac{t}{K\delta}) + r \log(\frac{t}{K\delta})\Big)$, 
then with a probability of at least $1 - \delta/3$, 
	\begin{align}\label{eq:v_update_incoherence_rkr_logt}
	\|\vi\| \leq \Ord\Big(\mu\,\eigmin \Big) \text{\;, and } \eigmin^* \leq 2\eigmin
	\end{align}	
	and %
	\begin{align}
	\sqrt{\frac{rK}{t}}\frac{\|\dv\|_F}{\sqrt{\eigmin}} 
	&\leq \Ord\Big(\sqrt{ \frac{\log(\frac{t}{K\delta})}{\log(\frac{1}\delta)} }  \sqrt{\frac{\eigmax^*}{\eigmin^*}} \|(\Id - \oU (\oU)^\top) U\|_F + 
	\SNRre
	\sqrt{\frac{r^2 \log(\frac{t}{K\delta})}{m}} 
	\Big) \label{eq:v_update_hf_rkr_logt}
	\\
	\sqrt{\frac{r K}{t}} \frac{\|\dv\|_\inftyone}{\sqrt{\eigmin}} &\leq \Ord\Big(\sqrt{ \frac{\log(\frac{t}{K\delta})}{\log(\frac{1}\delta)} }  \|(\Id - \oU (\oU)^\top) U\| \sqrt{\frac{\mu r K}{t}} + 
	\SNRre
	\sqrt{\frac{r^2 K \log(\frac{t}{K\delta})}{mt}}
	\Big) \label{eq:v_update_hinfty_rkr_logt}
	\end{align}
\end{lemma}

A proof is in Section~\ref{sec:v_update_rkr_logt_pf}.

\medskip\noindent\textbf{Update on $U$:}
Let $\VV,\cH, \hcH: \bR^{d \times r} \to \bR^{d \times r}$ be three linear operators, such that $\cV(U) = U \sum_{i \in \cT_k} \vi (\vi)^\top = U \VV$,  $\cH(U) = U \sum_{i \in \cT_k} \dvi (\vi)^\top$ and $\hcH(U) = \sum_{i \in \cT_k} \Si_2 U \dvi (\vi)^\top$, where $\dvi =\vi - Q^\invert \ovi$. $\cV$ is invertible and self-adjoint. Therefore $\cVih$ and $\cVh$ exist.
Let $\cI: \bR^{d \times r} \to \bR^{d \times r}$ be the identity mapping, such that $\cI(U) = U$.
\begin{align}
\hU- U^*Q &= \cA^\inv(\sum_{i \in \cT_k} \Si_2 \oU Q (Q^\invert \ovi - \vi) (\vi)^\top + \zi (\vi)^\top) \\
&= \cA^\inv(-\hcH(\oU Q) + \sum_{i \in \cT_k} \zi (\vi)^\top) \\
&= \cVih (\cVh \cA^\inv \cVh) \cVih(-\hcH(\oU Q) + \sum_{i \in \cT_k} \zi (\vi)^\top) \\
&= \cVih (\cI + \cE_1)(-(\cVih \cH + \cE_2)(\oU Q) + \cVih (\sum_{i \in \cT_k} \zi (\vi)^\top)) %
\end{align}
where $\cE_1 = (\cVih \cA \cVih)^\inv - \cI$ and $\cE_2 = \cVih \hcH - \cVih \cH$, and $F = \hU - \oU Q +  \cV^{-1} (\cH (\oU Q))$. Let $F = \hU - \oU Q +  \cV^{-1} (\cH (\oU Q))$

\begin{lemma}\label{lem:u_update_rkr_logt}
	Assume that the large probability event in Lemma~\ref{lem:v_update_rkr_logt} holds true. Then,
	\begin{align}
	\|\cVi \cH (\oU Q)\|_F &\leq \Ord\Big(\sqrt{\frac{\eigmax^*}{\eigmin^*} \log(\frac{t}{K})}  \|(\Id - \oU (\oU)^\top)U\|_F + \SNRre \sqrt{\frac{r^2 \log(\frac{t}{K \delta})}{m}} \Big)
	\label{lem:u_update_parallelerr_rkr_logt} 
	\end{align}
	and if %
	$mt \geq \Omega(\mu dr^2 K \log(t/K\delta))$, then with probability at least $1-\delta/3$
	\begin{align}
	&\|F\|_F 
	\leq 
	\Ord\Big(\sqrt{\frac{\eigmax^*}{\eigmin^*}\frac{\mu d r^2 K \log(\frac{t}{K\delta})}{mt}} \|(\Id - \oU (\oU)^\top)U\|_F + \sqrt{\frac{\mu d r^2 K \log(\frac{t}{K\delta}) \log(\frac{r}\delta)}{mt}} \SNRre \sqrt{\frac{r^2 \log(\frac{1}{\delta})}{m}} \Big) 
	\label{lem:u_update_orthoerr_rkr_logt}
	\end{align}
\end{lemma}
A proof is in Section~\ref{sec:u_update_rkr_logt_pf}.
\begin{lemma}\label{lem:qr_r_inverse_logt}
	If $\frac12 \leq \sigma_{\min}(Q)$, $\|F\|_F \leq \frac18$ and $\|\cV^{-1} (\cH (\oU Q))\|_F \leq \frac18$, then $R$ is invertible and $\|R^\invert\| \leq 4$.
\end{lemma}
A proof is in Section~\ref{sec:qr_r_inverse_logt_pf}. Clearly, from \eqref{lem:u_update_parallelerr_rkr_logt} and \eqref{lem:u_update_orthoerr_rkr_logt}, a sufficient condition for the above lemma is
\begin{align}
&\Ord\Big(\sqrt{\frac{\eigmax^*}{\eigmin^*} \log(\frac{t}{K})} \|(\Id - \oU (\oU)^\top)U\|_F + \SNRre \sqrt{\frac{r^2 \log(\frac{t}{K \delta})}{m}} \Big) \leq \frac18 \text{\;, and } \\
&\Ord\Big(\sqrt{\frac{\eigmax^*}{\eigmin^*}\frac{\mu d r^2 K \log(\frac{t}{K\delta})}{mt}} \|(\Id - \oU (\oU)^\top)U\|_F + \sqrt{\frac{\mu d r^2 K \log(\frac{t}{K\delta}) \log(\frac{r}\delta)}{mt}} \SNRre \sqrt{\frac{r^2 \log(\frac{1}{\delta})}{m}} \Big) 
\leq \frac18 
\end{align}
which can be satisfied with
\begin{align}
&\|(\Id - \oU (\oU)^\top)U\|_F \leq \Ord\Big(\sqrt{\frac{\eigmin^*}{\eigmax^*} \frac1{\log(t/K)} }\Big) \text{\,, \;\; } 
m \geq \Omega(\Big(\SNRre \Big)^2 {r^2 \log(\frac{t}{K \delta})} + {r^2 \log(\frac{1}{\delta})}\Big) \text{\;, and } \\
&mt \geq \Omega\Big(\mu dr^2 K \Big( 1 + \Big(\SNRre \Big)^2 \log(\frac{t}{K\delta}) \log(\frac{r}\delta) \Big) \Big)
\end{align}
Finally, we bound the Frobenius norm distance of the next iterate $\pU$ from the optimal $\oU$.
\begin{align}
\|(\Id - &\oU (\oU)^\top)\pU\|_F \\
&= \min_{Q^+} \|\pU - \oU Q^+\|_F \\
&\leq \|\hU R^\invert - \oU Q R^\invert + (\cVi \cH  (\oU Q)) R^\invert\| \\
&\leq \|\hU - \oU Q + \cVi \cH(\oU Q) \|_F \|R^\invert\| \\
&= \|F\|_F \|R^\invert\| \\
&\leq 
\Ord\Big(\sqrt{\frac{\eigmax^*}{\eigmin^*}\frac{\mu d r^2 K \log(\frac{t}{K\delta})}{mt}} \|(\Id - \oU (\oU)^\top)U\|_F + \sqrt{\frac{\mu d r^2 K \log(\frac{t}{K\delta}) \log(\frac{r}\delta)}{mt}} \SNRre \sqrt{\frac{r^2 \log(\frac{1}{\delta})}{m}} \Big) 
\end{align}

If 
\begin{align}
&mt \geq \Omega\Big(\mu dr^2 K {\frac{\eigmax^*}{\eigmin^*}} \Big( \log(\frac{t}{K\delta}) + \Big(\SNRre \Big)^2 \log^2(\frac{t}{K\delta}) \log(\frac{r}\delta) \Big) \Big)\text{\,, and } m \geq \Omega\Big(r^2 \log(\frac{1}\delta)\Big)
\end{align}
then,
\begin{align}
\|(\Id - \oU (\oU)^\top) \pU\|_F &\leq \frac12 \|(\Id - \oU (\oU)^\top)U\|_F + \min \Big( \frac38, \Ord\Big(\sqrt{\frac{\eigmin^*}{\eigmax^*} \frac1{\log(t/K)} } \Big)\Big)
\end{align}
Thus if $\|(\Id - \oU (\oU)^\top)U\|_F \leq \min \Big( \frac34, \Ord\Big(\sqrt{\frac{\eigmin^*}{\eigmax^*} \frac1{\log(t/K)} } \Big)\Big)$, then $\|(\Id - \oU (\oU)^\top)\pU\|_F \leq \min \Big( \frac34, \Ord\Big(\sqrt{\frac{\eigmin^*}{\eigmax^*} \frac1{\log(t/K)} } \Big)\Big)$. 

In the following lemma we prove that tasks subset used for each iteration, satisfy approximate incoherence.
\begin{lemma}[Shuffling and partition of tasks]\label{lem:shuffling_rkr_logt}
Let $\cT_k$ be the $k$-th subset ($k \in [K]$) of the $K$-way partition of the shuffled set of all $t$ tasks. If $t \geq \Omega(\mu^2 r^3 K \log(1/\delta))$, then with a probability of at least $1 - \delta/3$,
\begin{align}
\lambda_{1}(\sum_{i \in \cT_k} \ovi(\ovi)^\top) = \frac1K \Theta(\lambda_{1}((\ov)^\top \ov)) \;\; \text{ and } \;\; \lambda_{r}(\sum_{i \in \cT_k} \ovi(\ovi)^\top) = \frac1K \Theta(\lambda_{r}((\ov)^\top \ov)) \;,\;\;\; \text{ for all $r' \in [r]$}
\end{align}
where are $\lambda_{1}(\cdot)$ and $\lambda_{r}(\cdot)$ are the largest and smallest, respectively, eigenvalue operators of real-symmetric $r \times r$ matrix.
\end{lemma} 
A proof is in Section \ref{sec:shuffling_rkr_logt_pf}.

Therefore, using union-bound, we can un-roll the relation, between current iterate $U$ and the next iterate $\pU$, over $K$ iterations, starting from $\iU$ and ending at some $U$ iterations, to get
\begin{align}
\|(\Id - \oU (\oU)^\top) U\|_F &\leq \frac1{2^K} \|(\Id - \oU (\oU)^\top)\iU\|_F + 
\Ord\Big(\sqrt{\frac{\mu d r^2 K \log(\frac{t}{K\delta}) \log(\frac{r}\delta)}{mt}} \SNRre \sqrt{\frac{r^2 \log(\frac{1}{\delta})}{m}} \Big) 
\end{align}
with probability at least $1 - K \delta$.
Finally setting $K = \lceil \log_2(\frac{({\eigmin^*}/{\eigmax^*}) mt}{\mu dr^2}) \rceil$ and using 
$m \geq \Omega(r^2 \log(\frac1{\delta}))$
we get that, with a probability of at least $1 - K\delta$
\begin{align}
\|(\Id - \oU (\oU)^\top) U\|_F &\leq 
\Ord\Big(\SNRre \sqrt{\frac{\mu d r^2 K \log(\frac{t}{K\delta}) \log(\frac{r}\delta)}{mt}} \Big)
\end{align}

\subsection{Analysis of update on $V$}

\subsubsection{Proof of Lemma~\ref{lem:v_update_rkr_logt}}
\label{sec:v_update_rkr_logt_pf}
\begin{proof}[Proof of Lemma~\ref{lem:v_update_rkr_logt}]
	In this proof for brevity, we will first set that $\cT_k \leftarrow [t]$, $|\cT_k| = t/K \leftarrow t$, $\Si_1 \leftarrow \Si = \frac1m \sum_{j \in [m]} \xij (\xij)^\top$. This can be done due to the approximate equivalence of the subset $\cT_k$ by Lemma~\ref{lem:shuffling_rkr_logt}. Finally at the end of the analysis we will reset $\cT_k \leftarrow \cT_k$, $|\cT_k| = t/K \leftarrow t/K$, $\Si_1 \leftarrow \Si_1 = \frac2m \sum_{j \in [1, m/2]} \xij (\xij)^\top$.
	
	Recall the definition of $\vi$ from the update \eqref{eq:r_vupdate_logt}, and that $Q^{-1}$ is right inverse of $Q$, i.e.~$Q Q^{-1} = \Id$.
\begin{eqnarray}
\vi - Q^\invert \ovi &=&  (U^\top \Si U )^\inv(U^\top \Si (\oU Q-U) ) Q^\invert \ovi  + (U^\top \Si U )^\inv U^\top z^{(i)}  
\end{eqnarray}	
	We can use re-write the first term as,
	\begin{align}
	&\;\;\;\;\; (U^\top \Si U)^\inv U^\top \Si (\oU Q-U) Q^\invert \\ &= (U^\top \Si U)^\inv U^\top \Si (U U^\top + U_\perp U_\perp^\top) (\oU Q-U) Q^\invert \\
	&= U^\top (\oU Q-U) Q^\invert + (U^\top \Si U)^\inv U^\top \Si U_\perp U_\perp^\top  (\oU Q-U) Q^\invert \\
	&= - U^\top (\Id - \oU (\oU)^\top)^2 U Q^\invert + (U^\top \Si U)^\inv U^\top \Si U_\perp U_\perp^\top  \oU  \\
	&= - (U - \oU Q )^\top (U - \oU Q ) Q^\invert + (U^\top \Si U)^\inv U^\top \Si U_\perp U_\perp^\top  \oU
	\end{align}
	where we used the fact that $Q = (\oU)^\top U$. Therefore
	\begin{align}
	& \| \vi - Q^\invert \ovi \|  %
	\leq \nonumber \\ 
	& \|U - \oU Q\| \|(U - \oU Q) Q^\invert \ovi\| + 
	\| (U^\top \Si U)^\inv \| (\| U^\top \Si U_\perp U_\perp^\top \oU \ovi\| +
	\| U^\top z^{(i)}  \,\|)
	\end{align}
	If $m \geq \Omega( {r \log(t/\delta)} )$, then $\alpha = c \sqrt{\frac{r \log(27t/\delta)}{m}} \leq 1/2$ and by 
	Lemma~\ref{lem:v_update_rkr_matrices_logt}, with a probability of at least $1 - \delta$,
	\begin{align}
	\left.\begin{aligned}
	\| (U^\top \Si U )^\inv\big\| &\leq (1+2\alpha) \text{,  } \\
	\big\| U^\top \Si U_\perp U_\perp^\top \oU \ovi \big\|
	&\leq \alpha \|U_\perp^\top \oU \ovi \| \text{, \; and \;\; } \\
	\big\| U^\top z^{(i)}  \,\big\| &\leq \sigma \alpha\,,
	\end{aligned}\right\rbrace
	\text{ for all $i \in [t]$}
	\end{align}
	Now if $m \geq \Omega(r \log(1/\delta))$ and $\|\oU Q-U\| \leq \Ord\Big(\sqrt{ \frac{\log(\frac{t}\delta)}{\log(\frac{1}\delta)} }\Big)$, then 
	\begin{eqnarray}
	\| \vi - Q^\invert \ovi \|  & \leq & \Ord(\sqrt{ \frac{\log(\frac{t}\delta)}{\log(\frac{1}\delta)} } (\|(\oU Q-U) Q^{-1} \ovi \| + \|U_\perp^\top \oU \ovi\big\|) +  \sigma \sqrt{\frac{r \log(\frac{t}{\delta})}{m}}) 
	\label{eq:v_update_vdiff_rkr_logt}
	\end{eqnarray}
	Next we bound $\|\dv\|_F$, which by definition is $\|\dv\|_F = \sqrt{\sum_{i \in [t]} \| \dvi \|^2} = \sqrt{\sum_{i \in [t]} \| \vi - Q^\invert \ovi \|^2}$. Using \eqref{eq:v_update_vdiff_rkr_logt} and the fact that $(a^2 + b^2) \leq 2(a^2 + b^2)$ we get
	\begin{align}
	\|\dv\|^2_F&\leq \frac{\log(\frac{t}\delta)}{\log(\frac{1}\delta)} [\sum_{i \in \cT}  \Ord( \|(\oU Q-U)Q^\invert \ovi\|^2 + \|U_\perp^\top \oU \ovi \|^2)] + t(\sigma \sqrt{\frac{r \log(\frac{t}{\delta})}{m}})^2) \label{eq:v_update_hf1_rkr_logt}
	\end{align}	
	Clearly $\|Q\| = \|(\oU)^\top U\| \leq \|\oU\|\|U\| \leq 1$. If $\|(\Id - \oU (\oU)^\top)U \| \leq \|(\Id - \oU (\oU)^\top)U \|_F \leq \frac34$, then by using Lemma~\ref{lem:distance-relations}, $\|Q^\invert\| \leq 2$.	
	\begin{align}
	\sum_{i \in [t]} \|(\oU Q-U)Q^\invert \ovi\|^2 &= \sum_{i \in [t]} \tr((\ovi)^\top ((\oU Q-U)Q^\invert)^\top (\oU Q-U)Q^\invert \ovi) \\
	&= \tr((\oU Q-U)Q^\invert)^\top (\oU Q-U)Q^\invert) \sum_{i \in [t]} \ovi (\ovi)^\top ) \\
	&\leq \|(\oU Q-U)\|_F^2 \|Q^\invert\|^2 \Ord(\eigmax^*) (t/r) \\
	&\leq 4\|(\oU Q-U)\|_F^2 \Ord(\eigmax^*) (t/r)
	\end{align}
	Similarly we can use Lemma~\ref{lem:distance-relations}, to get
	\begin{align}
	\sum_{i \in [t]} \|U_\perp^\top \oU \ovi \|^2 &= \sum_{i \in [t]} \tr((\ovi)^\top (U_\perp^\top \oU )^\top U_\perp^\top \oU  \ovi) \\
	&= \tr((U_\perp^\top \oU)^\top (U_\perp^\top \oU) \sum_{i \in [t]} \ovi (\ovi)^\top ) \\
	&\leq \|U_\perp^\top \oU\|_F^2 \Ord(\eigmax^*) (t/r) \\
	&\leq \|(\oU Q-U)\|_F^2\Ord(\eigmax^*) (t/r)
	\end{align}
	Therefore substituting the above two inequalities into \eqref{eq:v_update_hf1_rkr_logt} and using the fact that $\sqrt{a + b} \leq \sqrt{a} + \sqrt{b}$ for all $0 \leq a, b$ we get
	\begin{align}
	\|\dv\|_F %
	&\leq \Ord(\sqrt{ \frac{\log(\frac{t}\delta)}{\log(\frac{1}\delta)} } \|\oU Q-U\|_F \sqrt{\eigmax^* (t/r) } + \sqrt{t} \sigma \sqrt{\frac{r \log(\frac{t}{\delta})}{m}}) \label{eq:v_update_hf2_rkr_logt}
	\end{align}
	Then as $\|(\Id - \oU (\oU)^\top)U\|_F \leq \Ord\Big(\sqrt{\frac{\eigmin^*}{\eigmax^*} \frac1{\log(t)} }\Big)$ and $m \geq \Omega\Big( \Big(\SNRre \Big)^2 r^2 \log(\frac{t}\delta)\Big)$, $ \|\dv\|_F  \leq  (1-\frac1{\sqrt{2}})  \sqrt{(t/r) \eigmin^*}$.
	Using $\|Q^\invert\| \leq 2$ in \eqref{eq:v_update_vdiff_rkr_logt} we also get that
	\begin{align}
	\|\dvi\| = \| \vi - Q^\invert \ovi \| 
	\leq \Ord(\sqrt{ \frac{\log(\frac{t}\delta)}{\log(\frac{1}\delta)} } \|(\oU Q-U)\| \|\ovi \| +  \sigma \sqrt{\frac{r \log(\frac{t}{\delta})}{m}}) \label{eq:v_update_hinf_rkr_logt}
	\end{align}
	By definition is $\|\dv\|_\inftyone = \max_{i \in [t]} \| \dvi \| = \max_{i \in [t]} \| \vi - Q^\invert \ovi \|$.
	Then as $\|(\Id - \oU (\oU)^\top)U\| \leq \|(\Id - \oU (\oU)^\top)U\|_F \leq \Ord\Big(\sqrt{\frac{\eigmin^*}{\eigmax^*} \frac1{\log(t)}}\Big) \leq \Ord(1)$, $m \geq \Omega\Big( \Big(\SNRre \Big)^2 r^2 \log(\frac{t}\delta)\Big) \geq \Omega\Big( \Big(\SNRre\Big)^2 r \log(\frac{t}\delta)\Big) $, $\|\dv\|_\inftyone \leq \Ord( \mu \eigmin^*)$.
	Now, using $ \|\dv\|_F  \leq  (1-\frac1{\sqrt{2}})  \sqrt{(t/r) \eigmin^*}$, $\|\dv\|_\inftyone \leq \Ord( \mu \eigmin^*)$, $\|Q\| \leq 1$ and $\frac12 \leq \sigma_{\min}(Q)$, by Lemma~\ref{lem:incoherence_logt}, we get the approximate incoherence relation for the intermediate $\v$
	\begin{align}
	\|\vi\| \leq \Ord\Big(\mu\,\eigmin\Big) \text{\;, and } \eigmin^* \leq 2\eigmin \label{eq:v_update_incoh_rkr_logt}
	\end{align}		
	Using this we bound $\|\dv\|_\inftyone$. Using the above incoherence relation and \eqref{eq:v_update_hinf_rkr_logt}, we get
	\begin{align}
	\sqrt{\frac{r}{t}} \frac{ \|\dv\|_\inftyone}{\sqrt{\eigmin}} \leq 2 \sqrt{\frac{r}{t}} \frac{ \|\dv\|_\inftyone}{\sqrt{\eigmin^*}} &\leq \Ord\Big( \sqrt{\frac{r}{t}} \sqrt{ \frac{\log(\frac{t}\delta)}{\log(\frac{1}\delta)} }  \|\oU Q-U\| \max_{i \in [t]} \frac{\|\ovi \|}{\sqrt{\eigmin^*}}  + 2\sqrt{\frac{r}{t}} \frac{2 c \sigma}{\sqrt{\eigmin^*}} \sqrt{\frac{r \log(\frac{27t}{\delta})}{m}} \\
	&\leq \Ord\Big( \sqrt{ \frac{\log(\frac{t}\delta)}{\log(\frac{1}\delta)} }  \sqrt{\frac{\mu r}{t}} \|\oU Q-U\| + \SNRre \sqrt{\frac{r^2 \log(\frac{t}{\delta})}{mt}} \Big)
	\end{align}
	Using \eqref{eq:v_update_incoh_rkr_logt} in \eqref{eq:v_update_hf2_rkr_logt}, we get
	\begin{align}
	\sqrt{\frac{r}{t}}\frac{\|\dv\|_F}{\sqrt{\eigmin}} \leq 2 \sqrt{\frac{r}{t}} \frac{\|\dv\|_F}{\sqrt{\eigmin^*}} &\leq \Ord\Big( \sqrt{ \frac{\log(\frac{t}\delta)}{\log(\frac{1}\delta)} }  \sqrt{\frac{\eigmax^*}{\eigmin^*}} \|(\Id - \oU (\oU)^\top)U\|_F + \SNRre
	\sqrt{\frac{r^2 \log(\frac{t}{\delta})}{m}} \Big)
	\end{align}	
	Finally, by resetting $\cT_k \leftarrow \cT_k$, $|\cT_k| = t/K \leftarrow t/K$, $\Si_1 \leftarrow \Si_1 = \frac2m \sum_{j \in [1, m/2]} \xij (\xij)^\top$, we obtain the desired result.
\end{proof}

\subsubsection{Supporting lemmas for the analysis of update on $V$}

Here we bound the linear operators in the $\vi$ update.
\begin{lemma}\label{lem:v_update_rkr_matrices_logt}
Let $\alpha = c \sqrt{\frac{r \log(27t/\delta)}{m}}$. With a probability of at least $1 - \delta$, the following are true for all $i \in [t]$
\begin{align}
\| (U^\top \Si U )^\inv\big\| &\leq (1+2\alpha) \text{,  } \\
\big\| (U^\top \Si (\oU Q-U) Q^\invert \ovi\big\| &\leq (\|(\Id - \oU (\oU)^\top)U\| + \alpha) \|(\oU Q-U) Q^\invert \ovi\big\| \\ &\leq (1+\alpha) \|(\oU Q-U) Q^\invert \ovi\big\| \\
\big\| U^\top \Si U_\perp U_\perp^\top \oU \ovi\big\| &\leq \alpha \big\| U_\perp^\top \oU \ovi\big\| \text{, and } \\
\big\| U^\top z^{(i)}  \,\big\| &\leq \sigma \alpha
\end{align}
\end{lemma}
\begin{proof}[Proof of Lemma~\ref{lem:v_update_rkr_matrices_logt}]
Let $i \in [t]$.
	
Let $\mathcal{S} = \{v \in \bR^{r} \,|\, \|v\| = 1\}$ be the set of all real vectors of dimension $r$ with unit Euclidean norm. For $\epsilon \leq 1$, there exists an $\epsilon$-net, $N_\epsilon \subset \mathcal{S}$, of size $(1 + 2/\epsilon)^{r}$ with respect to the Euclidean norm~\citep[Lemma 5.2]{vershynin2010introduction}. That is for any $v' \in \mathcal{S}$, there exists some $v \in N_\epsilon$ such that $\|v'-v\|_F \leq \epsilon$.

Consider a $v \in N_\epsilon$, such that $\|v\|_F = 1$. Now we will prove with high-probability that $\big\langle ((U^\top \Si U) - \Id)v, v \big\rangle$ is small. Consider the the following quadratic form
\begin{align}
v^\top (U^\top \Si U) v = \frac1m \sum_{j \in [m]} \tr(v^\top (U^\top \xij (\xij)^\top U) v) &= \frac1m \sum_{j \in [m]} \tr((\xij)^\top U v v^\top U^\top \xij )
\end{align}
$\xij \sim {\cal N}(0,{\mathbf I}_{d\times d}$) are i.i.d.~standard Gaussian random vectors. We will use Hanson-Wright inequality (Lemma ~\ref{lem:hanson-wright}) to prove that the above quadratic form concentrates around its mean. In  Lemma~\ref{lem:gauss_pseudo_inner} (which is a straightforward Corollary of Hanson-Wright inequality), by setting $a\leftarrow U v, b  \leftarrow U v$, we get that with a probability of at least $1 - \delta$
\begin{equation}
\bigg|v^\top ((U^\top \Si U) - \Id)v \bigg| \leq c \max\bigg(\sqrt{\frac{ \log(1/\delta)}{m}}, \frac{\log(1/\delta)}{m} \bigg) := \Delta_\epsilon
\end{equation}
For brevity, let $E = (U^\top \Si U) - \Id$. Notice that $E$ is a real symmetric matrix, therefore it has an eigen decomposition. Then, let $v' \in \mathcal{S} \subset \bR^{r}$ be the largest ``eigenvector'' of $E$, such that $(v')^\top E v' = \|E\| = \max_{\|\widetilde{v}\| = 1} \widetilde{v}^\top E \widetilde{v} = \max_{\|\widetilde{v}\| = \|\widetilde{v}'\|_F = 1} \widetilde{v}^\top E \widetilde{v}'$. Then there exists some $v \in N_\epsilon$ such that $\|v'-v\| \leq \epsilon$. 
\begin{align}
\|E\|_F = (v')^\top E v &= v^\top E v + (v'-v)^\top E v+  (v')^\top E (v' - v) \\
&\leq v^\top E v + \|v'-v\|\|E\| \|v\| +  \|v'\| \|E\| \|v' - v\|   \\
&\leq v^\top E v  + 2 \epsilon \|E\| 
\end{align}
Re-arranging and setting $\epsilon=1/4$, and $c \gets 2c$, we get
\begin{align}
\|(U^\top \Si U) - \Id\| = \|E\| \leq \Delta_{\frac14} = \Delta.
\end{align}
where $\Delta = c \max\bigg(\sqrt{\frac{r\, \log(9/\delta)}{m}}, \frac{r\,\log(9/\delta)}{m} \bigg)$. 
If $m \geq \max(1, 4c^2) {r \log(27t/\delta)}$, then $\Delta \leq \alpha \leq 1/2$.

Thus with a probability of at least is is also implies that
\begin{align}
\|(U^\top \Si U)^\inv\| = (\sigma_{\min}(U^\top \Si U))^{-1} \leq \frac1{1 - \alpha} \leq 2.
\end{align}

Using similar arguments we can also prove that with a probability of at least $1 - \delta$
\begin{align}
\big\| (U^\top \Si (\oU Q-U) Q^\invert \ovi\big\| &\leq \|U^\top (\oU Q-U) Q^\invert \ovi\big\| + \alpha \|(\oU Q-U) Q^\invert \ovi\big\| \\
&\leq \|U^\top (\Id - \oU (\oU)^\top) U Q^\invert \ovi\big\| + \alpha \|(\oU Q-U) Q^\invert \ovi\big\| \\
&\leq \|U^\top (\Id - \oU (\oU)^\top)^2 U Q^\invert \ovi\big\| + \alpha \|(\oU Q-U) Q^\invert \ovi\big\| \\
&\leq \|U^\top (\Id - \oU (\oU)^\top) (\oU Q-U) Q^\invert \ovi\big\| + \alpha \|(\oU Q-U) Q^\invert \ovi\big\| \\
&\leq \|(\Id - \oU (\oU)^\top) U\| \|(\oU Q-U)Q^\invert \ovi\big\| + \alpha \|(\oU Q-U) Q^\invert \ovi\big\| \\
&\leq (\|(\Id - \oU (\oU)^\top)U\| + \alpha) \|(\oU Q-U) Q^\invert \ovi\big\| \\ &\leq (1+\alpha) \|(\oU Q-U) Q^\invert \ovi\big\|\,,
\end{align}
Using similar arguments we can also prove that with a probability of at least $1 - \delta$
\begin{align}
\big\| U^\top \Si U_\perp U_\perp^\top \oU \ovi\big\| &\leq \alpha \big\| U_\perp^\top \oU \ovi\big\|
\end{align}
and with a probability of at least $1 - \delta$
\begin{align}
\big\| U^\top z^{(i)}  \,\big\| &\leq \sigma \alpha
\end{align}

Finally setting $\delta \leftarrow \delta/3/t$ and taking the union bound over three bounds over all the tasks in $[t]$ gets us the desired result.
\end{proof}
Here we prove the approximate incoherence of the intermediate $\v$ and the spectrum of intermediate $\VV$.
	\begin{lemma}[Incoherence of intermediate $\vi$]\label{lem:incoherence_logt} If 
	$ \|\dv\|_F  \leq  (1-\frac1{\sqrt{2}})  \sqrt{(t/r) \eigmin((r/t) W^*)}$, %
	$\|\dv\|_\inftyone^2 \leq \Ord(\mu \eigmin((r/t) W^*)) $, %
	$\|Q\| \leq 1$ and $\frac12 \leq \sigma_{\min}(Q)$, and \eqref{eq:v_update_hf2_rkr_logt} and \eqref{eq:v_update_hinf_rkr_logt} are true,
	then
	\begin{align}
	\|\vi\| \leq \Ord\Big(\mu\,\eigmin((r/t)W)\Big) \text{\;, and } \eigmin((r/t)W^*) \leq 2\eigmin((r/t)W)
	\end{align}
\end{lemma}

\begin{proof}[Proof of Lemma~\ref{lem:incoherence_logt}]
\begin{align}
\|\vi\| &\leq  \|Q^{-1}\ovi\| + \|\vi - Q^{-1} \ovi\| \leq  2\|\ovi\| + \|\dvi\| \label{eq:v_update_v_norm_ub_rkr_logt} \\ 
\implies \|\vi\|^2 &\leq \Ord(\|\ov\|_\inftyone^2) + \Ord( \|\dv\|_\inftyone^2) \leq \Ord\Big( \mu \eigmin((r/t) W^*) \Big)
\end{align}
where the second inequality use the definition $\dvi = \vi - Q^{-1} \ovi$ and $\|Q^{-1}\| \leq 2$ (as $\sigma_{\min}(Q) \geq \frac12$), the third inequality use the fact that $a+b \leq 2 a^2 + 2 b^2$ an \eqref{eq:v_update_hinf_rkr_logt}, and the final inequality uses $\|\dv\|_\inftyone \leq \|\v\|_\inftyone$.

Notice that $W = \v^T \v$ and $W^* = (\ov)^T \ov$. Thus $\sqrt{\eigmin((r/t) W)} =  \sqrt{(r/t)} \sigmin(V)$ and $\sqrt{\eigmin((r/t) W^*)} =  \sqrt{(r/t)} \sigmin(W^*)$, and both $W$ and $W^*$ are positive semi-definite (PSD). Similarly, using $\sigma_{\min}(Q^{-1}) = \sigma_{\min}(((\oU)^\top U)^{-1}) \geq 1$ and Lemma~\ref{lem:eig_prod} we can get that 
\begin{align}
\sqrt{\eigmin((r/t) W^*)} \leq \sqrt{\sigma_{\min}^2(Q^{-1}) \eigmin((r/t) W^*)} \leq \sqrt{(r/t) \eigmin(Q^{-1} (\ov)^T \ov Q^{-\top})} \leq \sqrt{(r/t)} \sigmin(\ov Q^{-T})
\end{align}
Therefore, instead of analyzing the relation between $\eigmin(W)$ and $\eigmin(W^*)$, we can analyze the relation between  $\sigmin(V)$ and $\sigmin(V^*)$. Notice that $\ov Q^{-T} = \v + \ov Q^{-T} - \v$. Then by Weyl's inequality (Lemma~\ref{lem:weyls}, by setting $A \leftarrow \ov Q^{-T}$, $B \leftarrow \v$, and $C \leftarrow \ov Q^{-T} - \v$) we get that
\begin{align}
\sqrt{\eigmin((r/t) W^*)} 
\leq  \sqrt{(r/t)} \sigmin(\ov Q^{-T}) 
&\leq  \sqrt{(r/t)} \sigmin(\v) +  \sqrt{(r/t)} \|\v - \ov Q^{-T}\| \\
&\leq \sqrt{\eigmin((r/t)W)}  +  \sqrt{(r/t)}  \|\dv\| \\ 
&\leq \sqrt{\eigmin((r/t) W)} + \sqrt{(r/t)} \|\dv\|_F \\
&\leq \sqrt{\eigmin((r/t) W)} + (1-\frac1{\sqrt{2}}) \sqrt{\eigmin((r/t) W^*)}
\end{align}
where %
the last inequality uses $ \|\dv\|_F  \leq  (1-\frac1{\sqrt{2}}) \sqrt{(t/r) \eigmin((r/t) W^*)}$. %
Finally we get the desired result by re-arranging the terms.
\end{proof}

\subsection{Analysis of update on $U$}

\subsubsection{Proof of Lemma~\ref{lem:u_update_rkr_logt}}
\label{sec:u_update_rkr_logt_pf}
\begin{proof}[Proof of Lemma~\ref{lem:u_update_rkr_logt}]
	In this proof for brevity, we will first set that $\cT_k \leftarrow [t]$, $|\cT_k| = t/K \leftarrow t$, $\Si_2 \leftarrow \Si = \frac1m \sum_{j \in [m]} \xij (\xij)^\top$. This can be done due to the approximate equivalence of the subset $\cT_k$ by Lemma~\ref{lem:shuffling_rkr_logt}. Finally at the end of the analysis we will reset $\cT_k \leftarrow \cT_k$, $|\cT_k| = t/K \leftarrow t/K$, $\Si_2 \leftarrow \Si_2 = \frac2m \sum_{j \in [m/2+1, m]} \xij (\xij)^\top$.

	Recall that
	\begin{align}
	\hU- U^*Q %
	&= \cVih (\cI + \cE_1)(-(\cVih \cH + \cE_2)(\oU Q) + \cVih (\sum_{i \in [t]} \zi (\vi)^\top)) %
	\end{align}
	where $\cE_1 = (\cVih \cA \cVih)^\inv - \cI$ and $\cE_2 = \cVih \hcH - \cVih \cH$, and $F = \hU - \oU Q +  \cV^{-1} (\cH (\oU Q))$. Therefore
	
	\begin{align}
	\|F\|_F %
	&\leq %
	\|\cVih\|_F (\|\cE_1\|_F \|\cVih \cH (\oU Q)\|_F + \|\cI + \cE_1\|_F(\|\cE_2(\oU Q)\|_F + \|  \cVih( \sum_{i \in [t]} \zi (\vi)^\top)) \|_F)) 
	\label{eq:u_update_F_rkr_logt}
	\end{align}
	
	We can trivially bound $\|\cVih\|_F$ as follows. For all $\|U\|_F = 1$, the following is true.
	\begin{align} \label{eq:u_update_winv_rkr_logt}
	\|\cVih(U)\|_F = \|U\VVih\|_F \leq \|U\|_F \|\VVih\| \leq \sqrt{\frac{r/t}{\eigmin}}
	\end{align}

		$\Omega({\mu dr^2 \log(1/\delta)}) \leq {mt}$ and approximate incoherence of intermediate $V$ \eqref{eq:v_update_incoherence_rkr_logt} implies that $\Omega(dr \frac{\|\v\|_\inftyone^2 }{\eigmin(W)/t} \log(1/\delta) )\leq\Omega({\mu dr^2 \log(1/\delta)}) \leq {mt}$, then by	Lemma~\ref{lem:tail5} we have that, with a probability of at least $1 - \delta/3$
		\begin{align}
		\|\cE_1\|_F 
		\leq 3c \sqrt{\frac{dr\,\|\v\|_\inftyone^2 \log(27/\delta)}{m \, \eigmin(W)}}
		\leq 3c \sqrt{\frac{\mu d r^2 \log(27/\delta)}{mt}} 
		\leq \frac12 
		\end{align}
		This also implies that
		\begin{align} \label{eq:u_update_operator_rkr_logt}
		\|\cI + \cE_1\|_F \leq \|\cI\| + \|\cE_1\|_F \leq 1 + \Delta \leq \frac32
		\end{align}
	By Lemma~\ref{lem:tail6}, 
	\begin{align} \label{eq:u_update_hexp_rkr_logt}
	\|(\cVih \cH ) (\oU Q)\|_F \leq \|\dv\|_F
	\end{align}
	and with a probability of at least $1 - \delta/3$
	\begin{align}
	\|\cE_2 (\oU Q)\|_F &\leq c (\min(\|\dv\|_F \frac{\|\v\|_\inftyone}{\sqrt{\eigmin(W)}},\|\dv\|_\inftyone )\sqrt{\frac{dr\,\log(15/\delta)}{m}} + \|\dv\|_\inftyone \frac{\|\v\|_\inftyone}{\sqrt{\eigmin(W)}} \frac{dr\,\log(15/\delta)}{m})
	\end{align}
	Using the approximate incoherence of $\v$ \eqref{eq:v_update_incoherence_rkr_logt} in the above inequality, we get that
	\begin{align} \label{eq:u_update_hvar_rkr_logt}
	\|\cE_2 (\oU Q)\|_F &\leq c (\min(\|\dv\|_F \sqrt{\frac{\mu r}{t}},\|\dv\|_\inftyone )\sqrt{\frac{dr\,\log(15/\delta)}{m}} + \|\dv\|_\inftyone 
	\sqrt{\frac{\mu r}{t}} \cdot \frac{dr\,\log(15/\delta)}{m})
	\end{align}
	By Lemma~\ref{lem:tail8} with a probability of at least $1 - \delta/3$
	\begin{align}
	\| \sum_{i \in [t]} \cVih (\zi (\vi)^\top)) \|_F &\leq 
	\Ord\Big(\sigma \sqrt{\frac{dr}{m} \log\Big(\frac{t}\delta\Big)  \log\Big(\frac{r}\delta\Big)} \Big)
	\label{eq:u_update_noise_rkr_logt}
	\end{align}
	Finally taking union bound over the above results and using Lemma~\ref{lem:v_update_rkr_logt}, we can bound each of the terms constituting $F$.
	Using \eqref{eq:u_update_winv_rkr_logt}, \eqref{eq:u_update_hexp_rkr_logt} and \eqref{eq:v_update_hf_rkr_logt} (recall that we set $t \leftarrow t/K$) we get
	\begin{align}
	\|\cVi \cH (\oU Q)\|_F &\leq \|\cVih\|_F \|\cVih \cH (\oU Q)\|_F \\
	&\leq  \sqrt{\frac{r}{t}} \frac{\|\dv\|_F}{\sqrt{\eigmin}} 
	\leq \Ord\Big(\sqrt{\frac{\eigmax^*}{\eigmin^*}} \sqrt{ \frac{\log(\frac{t}\delta)}{\log(\frac{1}\delta)} }  \|(\Id - \oU (\oU)^\top)U\|_F + \SNRre \sqrt{\frac{r^2 \log(\frac{t}{\delta})}{m}} \Big)  \label{eq:u_update_F1_rkr_logt}
	\end{align}
	Using \eqref{eq:u_update_winv_rkr_logt}, \eqref{eq:u_update_operator_rkr_logt},
	\eqref{eq:u_update_hexp_rkr_logt}, and \eqref{eq:v_update_hf_rkr_logt} we get
	\begin{align} 
	&\;\;\;\;\; \|\cVih\|_F \|\cE_1\|_F \|\cVih \cH (\oU Q)\|_F \\
	&\leq  \Ord\Big(\sqrt{\frac{\mu d r^2 \log(\frac1{\delta})}{mt}} \sqrt{\frac{r}{t}}  \frac{\|\dv\|_F}{\sqrt{\eigmin}}\Big) \\
	&\leq \Ord\Big(\sqrt{\frac{\eigmax^*}{\eigmin^*}\frac{\mu d r^2 \log(\frac{t}\delta)}{mt}} \|(\Id - \oU (\oU)^\top)U\|_F + \sqrt{\frac{\mu d r^2 \log(\frac{1}\delta)}{mt}} \SNRre \sqrt{\frac{r^2 \log(\frac{t}{\delta})}{m}} \Big) \label{eq:u_update_F2_rkr_logt}
	\end{align}
	Using  \eqref{eq:u_update_winv_rkr_logt}, \eqref{eq:u_update_operator_rkr_logt}, \eqref{eq:u_update_hvar_rkr_logt}, \eqref{eq:v_update_hf_rkr_logt} and \eqref{eq:v_update_hinfty_rkr_logt} we get
	\begin{align}
	&\;\;\;\;\; \|\cVih\|_F \|\cI + \cE_1\|_F(\|\cE_2(\oU Q)\|_F \\
	&\leq  \Ord\Big(\sqrt{\frac{r}{t}} \min\Big(\frac{\|\dv\|_F}{\sqrt{\eigmin}} \sqrt{\frac{\mu r}{t}},\frac{\|\dv\|_\inftyone}{\sqrt{\eigmin}}\Big)\sqrt{\frac{dr\,\log(\frac1\delta)}{m}} + \sqrt{\frac{r}{t}} \frac{\|\dv\|_\inftyone}{\sqrt{\eigmin}}
	\sqrt{\frac{\mu r}{t}} \frac{dr\,\log(\frac1\delta)}{m} \Big) \\
	&\leq  \Ord\Big(\min\Big(\sqrt{\frac{\eigmax^*}{\eigmin^*}\frac{\mu d r^2 \log(\frac{t}\delta)}{mt}} \|(\Id - \oU (\oU)^\top)U\|_F + \sqrt{\frac{\mu d r^2 \log(\frac{1}\delta)}{mt}} \SNRre \sqrt{\frac{r^2 \log(\frac{t}{\delta})}{m}},\\
	&\;\;\;\;\;\;\;\;\;\;\;\;\;\;\;\;\;\;\;\;\;
	\sqrt{\frac{\mu d r^2 \log(\frac{t}\delta)}{mt}} \|(\Id - \oU (\oU)^\top)U\| + \sqrt{\frac{d r \log(\frac{1}\delta)}{m}} \SNRre \sqrt{\frac{r^2 \log(\frac{t}{\delta})}{mt}}\Big) + \\
	&\;\;\;\;\;\;\;\;\;\;\;\; \frac{\mu d r^2 \log(\frac{t}\delta)}{mt} \|(\Id - \oU (\oU)^\top)U\| + \frac{\sqrt{\mu} d r \sqrt{r} \log(\frac{1}\delta)}{m\sqrt{t}} \SNRre \sqrt{\frac{r^2 \log(\frac{t}{\delta})}{mt}}\Big)  \label{eq:u_update_F3_rkr_logt}
	\end{align}
	Using  \eqref{eq:u_update_winv_rkr_logt}, \eqref{eq:u_update_operator_rkr_logt}, \eqref{eq:u_update_noise_rkr_logt}, and \eqref{eq:v_update_incoherence_rkr_logt} we get
	\begin{align}
	\|\cVih\|_F \|\cI + \cE_1\|_F \| \sum_{i \in [t]} \cVih (\zi (\vi)^\top)) \|_F \leq \Ord\Big(\SNRre \sqrt{\frac{dr^2 \log(\frac{t}\delta)  \log(\frac{r}\delta)}{mt} } \Big)
	\label{eq:u_update_F4_rkr_logt}
	\end{align}
	Substituting \eqref{eq:u_update_F1_rkr_logt}, \eqref{eq:u_update_F2_rkr_logt}, \eqref{eq:u_update_F3_rkr_logt}, and \eqref{eq:u_update_F4_rkr_logt} in \eqref{eq:u_update_F_rkr_logt} we get
	\begin{align}
	\|F\|_F &\leq \|\cVih\|_F (\|\cE_1\|_F \|\cVih \cH (\oU Q)\|_F + \|\cI + \cE_1\|_F(\|\cE_2(\oU Q)\|_F + \| \sum_{i \in [t]} \cVih (\zi (\vi)^\top)\|_F) )\\
	&\leq \Ord\Big(\sqrt{\frac{\eigmax^*}{\eigmin^*}\frac{\mu d r^2 \log(\frac{t}\delta)}{mt}} \|(\Id - \oU (\oU)^\top)U\|_F + \sqrt{\frac{\mu d r^2 \log(\frac{1}\delta)}{mt}} \SNRre \sqrt{\frac{r^2 \log(\frac{t}{\delta})}{m}} \Big) + \\
	&\;\;\;\;\; \Ord\Big(\frac{\mu d r^2 \log(\frac{t}\delta)}{mt} \|(\Id - \oU (\oU)^\top)U\| + \frac{\sqrt{\mu} d r \sqrt{r} \log(\frac{1}\delta)}{mt} \SNRre \sqrt{\frac{r^2 \log(\frac{t}{\delta})}{m}}\Big)  + \Ord\Big(\SNRre \sqrt{\frac{dr^2 \log(\frac{t}\delta)  \log(\frac{r}\delta)}{mt} } \Big) \\
	&\leq \Ord\Big(\sqrt{\frac{\eigmax^*}{\eigmin^*}\frac{\mu d r^2 \log(\frac{t}\delta)}{mt}} \|(\Id - \oU (\oU)^\top)U\|_F + \sqrt{\frac{\mu d r^2 \log(\frac{1}\delta)}{mt}} \SNRre \sqrt{\frac{r^2 \log(\frac{t}{\delta})}{m}} \Big) + \Ord\Big(\SNRre \sqrt{\frac{dr^2 \log(\frac{t}\delta)  \log(\frac{r}\delta)}{mt} } \Big) \\
	&\leq \Ord\Big(\sqrt{\frac{\eigmax^*}{\eigmin^*}\frac{\mu d r^2 \log(\frac{t}\delta)}{mt}} \|(\Id - \oU (\oU)^\top)U\|_F + \sqrt{\frac{\mu d r^2 \log(\frac{t}\delta) \log(\frac{r}\delta)}{mt}} \SNRre \sqrt{\frac{r^2 \log(\frac{1}{\delta})}{m}} \Big) %
	\end{align}
where the second-last inequality used the fact that $mt \geq \Omega(\mu dr^2 \log(\frac{t}{\delta}))$.
Finally, by resetting $\cT_k \leftarrow \cT_k$, $|\cT_k| = t/K \leftarrow t/K$, $\Si_2 \leftarrow \Si_2 = \frac2m \sum_{j \in [m/2+1, m]} \xij (\xij)^\top$, we obtain the desired result.
\end{proof}

\subsubsection{Supporting lemmas for the analysis of update on $U$}

\begin{lemma}
	\label{lem:tail5}
	If 
	$\max(1, 4c^2) dr \frac{\|\v\|_\inftyone^2 }{\eigmin(W)/t} \log(27/\delta) \leq {mt}$,
	then with a probability of at least $1-\delta/3$, 
	\begin{align}
	\|\cE_1\|_F \leq 
	3c \sqrt{\frac{dr\,\|\v\|_\inftyone^2 \log(27/\delta)}{m \, \eigmin(W)}}
	\end{align}
\end{lemma}
\begin{proof}[Proof of Lemma~\ref{lem:tail5}]
	Let $\mathcal{S}_F = \{U \in \bR^{d \times r} \,|\, \|U\|_F = 1\}$ be the set of all real matrices of dimensions $d \times r$ with unit Frobenius norm. For $\epsilon \leq 1$, there exists an $\epsilon$-net, $N_\epsilon \subset \mathcal{S}_F$, of size $(1 + 2/\epsilon)^{dr}$ with respect to the Frobenius norm~\citep[Lemma 5.2]{vershynin2010introduction}. That is for any $U' \in \mathcal{S}_{F}$, there exists some $U \in N_\epsilon$ such that $\|U'-U\|_F \leq \epsilon$.
	
	Consider a $U \in N_\epsilon$, such that $\|U\|_F = 1$. Now we will prove with high-probability that $\big\langle (\cVih \cA \cVih - \cI)(U), U \big\rangle$ is small. Consider the the following quadratic form
	\begin{align}
	\big\langle (\cVih \cA \cVih)(U), U \big\rangle &= \Big\langle \sum_{i \in [t]} \Si U \VVih \vi (\vi)^\top \VVih, U\Big\rangle \\ 
	&= \sum_{i \in [t]} \frac1m \sum_{j \in [m]}  (\xij)^\top (U \VVih \vi (\vi)^\top \VVih U^\top) \xij
	\end{align}
	where $\Si = \frac1m \sum_{j \in [m]}  \xij (\xij)^\top$ and $\xij \sim {\cal N}(0,{\mathbf I}_{d\times d}$) are i.i.d.~standard Gaussian random vectors and $\VV = \sum_{i \in [t]} \vi (\vi)^\top$ is rank-$r$ matrix. We will use Hanson-Wright inequality (Lemma~\ref{lem:hanson-wright}) to prove that the above quadratic form concentrates around its mean. Notice that the the expectation of $\big\langle (\cVih \cA \cVih)(U), U \big\rangle $ is $\Ip{\cI(U)}{U}$.
	\begin{align}\label{eq:tail5_expect}
	\sum_{i \in [t]} \bE \Big[ \Big\langle \Si U \VVih \vi (\vi)^\top \VVih , U\Big\rangle \Big] &= \Big\langle U \VVih  \sum_{i \in [t]} \vi (\vi)^\top \VVih, U \Big\rangle = \Ip{U}{U} = \|U\|_F^2 = 1\,.
	\end{align}
	We will also need the following bounds to apply the Hanson-Wright inequality. Recall that $\|\v\|_\inftyone = \max_{i \in [t]} \|\vi\|$. Then,
	\begin{align}
	\max_{i \in [t]} \| U \VVih \vi (\vi)^\top \VVih U^\top\| = \max_{i \in [t]} \| U \VVih \vi \|^2
	&\leq \max_{i \in [t]} \| U\|^2 \|\VV\|^2 \|\vi\|^2
	\leq \frac{\|\v\|_\inftyone^2}{\eigmin(\VV)}\,\label{eq:tail5_max}
	\end{align}
	Also note that,
	\begin{align}
	\sum_{i \in [t]} \| U \VVih \vi (\vi)^\top \VVih U^\top\|_F^2
	&= \sum_{i \in [t]} \| U \VVih \vi\|^4 \\
	&= \max_{i \in [t]} \| U \VVih \vi\|^2 \sum_{i \in [t]} \Ip{U \VVih \vi}{U \VVih \vi} \\
	&\leq \frac{\|\v\|_\inftyone^2}{\eigmin(\VV)}
	\end{align}
	where the last inequality used \eqref{eq:tail5_expect} and \eqref{eq:tail5_max}. 
	Then by Hanson-Wright inequality (Lemma~\ref{lem:hanson-wright}), with probability at least $1 - \delta/|N_\epsilon|$
	\begin{align}
	\big| \big\langle (\cVih \cA \cVih - \cI)(U), U \big\rangle \big| = \big| \Big\langle \sum_{i \in [t]} \frac1m \sum_{j \in [m]} \xij (\xij)^\top U \VVih \vi (\vi)^\top \VVih, U\Big\rangle - \Ip{U}{U} \big| \leq \Delta_\epsilon
	\end{align}
	where $\Delta_\epsilon = c \max(\sqrt{\frac{\|\v\|_\inftyone^2 \log(|N_\epsilon|/\delta)}{m \, \eigmin(W)}}, \frac{\|\v\|_\inftyone^2 \log(|N_\epsilon|/\delta)}{m \, \eigmin(W)})$. Taking union bound over all $U \in N_\epsilon$ implies that with probability at least $1 - \delta$
	\begin{align}
	\big| \big\langle (\cVih \cA \cVih - \cI)(U), U \big\rangle \big| \leq \Delta_\epsilon \, \text{ , \;for all } U \in N_\epsilon\,.
	\end{align}
	
	For brevity, let $\cE_1' (U) = (\cVih \cA \cVih - \cI)(U)$. Notice that $\cE_1'$ is self-adjoint, therefore it has an eigen decomposition with respect to the Frobenius norm. Then, let $U' \in \mathcal{S}_F \subset \bR^{d \times r}$ be the largest ``eigenmatrix'' of $\cE_1$, such that $\Ip{\cE_1'(U)}{U} = \|\cE_1'\|_F = \max_{\|\widetilde{U}\|_F = 1} \Ip{\cE_1'(\widetilde{U})}{\widetilde{U}} = \max_{\|\widetilde{U}\|_F = \|\widetilde{U}'\|_F = 1} \Ip{\cE_1'(\widetilde{U})}{\widetilde{U}'}$. Then there exists some $U \in N_\epsilon$ such that $\|U'-U\|_F \leq \epsilon$. 
	\begin{align}
	\|\cE_1'\|_F = \Ip{\cE_1' (U')}{U'} &= \Ip{\cE_1' (U)}{U} + \Ip{\cE_1' (U' - U)}{U} +  \Ip{\cE_1' (U')}{U' - U} \\
	&\leq \Ip{\cE_1' (U)}{U} + \|\cE_1'\|_F \|U' - U\|_F (\|U\|_F +  \|U'\|_F)  \\
	&\leq \Ip{\cE_1' (U)}{U} + 2 \epsilon \|\cE_1'\|_F 
	\end{align}
	Re-arranging and setting $\epsilon=1/4$, and $c \gets 2c$, we get
	\begin{align}
	\|\cVih \cA \cVih - \cI\|_F = \|\cE_1'\|_F \leq \Delta_{\frac14} = \Delta.
	\end{align}
	where $\Delta = c \max\Big(\sqrt{\frac{dr\,\|\v\|_\inftyone^2 \log(9/\delta)}{m \, \eigmin(W)}}, \frac{dr\,\|\v\|_\inftyone^2 \log(9/\delta)}{m \, \eigmin(W)}\Big)$. 
	
	For brevity, let $\hcA (U) = (\cVih \cA \cVih)(U)$. Notice that $\hcA$ is self-adjoint, therefore it has an eigen decomposition with respect to the Frobenius norm. Then, let $U' \in \mathcal{S}_F \subset \bR^{d \times r}$ be the smallest ``eigenmatrix'' of $\hcA$, such that $\Ip{\hcA(U)}{U} = \lambda_{\min}(\hcA) = \min_{\|\widetilde{U}\|_F = 1} \Ip{\hcA(\widetilde{U})}{\widetilde{U}} = \min_{\|\widetilde{U}\|_F = \|\widetilde{U}'\|_F = 1} \Ip{\hcA(\widetilde{U})}{\widetilde{U}'}$. Then there exists some $U \in N_\epsilon$ such that $\|U'-U\|_F \leq \epsilon$. 
	\begin{align}
	\lambda_{\min}(\hcA) = \big\langle\hcA (U'), U' \big\rangle &= \Ip{\cI(U)}{U} + \big\langle (\hcA-\cI) (U), U \big\rangle + \big\langle\hcA (U' - U) , U \big\rangle +  \big\langle \hcA (U'), U' - U \big\rangle \\
	&\geq 1 - \big|\big\langle (\hcA-\cI) (U), U \big\rangle \big| - \lambda_{\min}(\hcA) \|U' - U\|_F (\|U\|_F + \|U'\|_F )\\
	&\geq 1 - \Delta_\epsilon - 2 \epsilon \lambda_{\min}(\hcA)
	\end{align}
	Re-arranging and setting $\epsilon=1/4$, and $c \gets 2c$, we get that $\lambda_{\min}(\hcA) \geq \frac23 (1 - \Delta)$. Therefore,
	\begin{align}
	\|(\cVih \cA \cVih)^\inv\|_F = \frac1{\lambda_{\min}(\hcA)} \leq \frac{3}{2(1-\Delta)}.
	\end{align}
	where $\Delta = c \max\Big(\sqrt{\frac{dr\,\|\v\|_\inftyone^2 \log(9/\delta)}{m \, \eigmin(W)}}, \frac{dr\,\|\v\|_\inftyone^2 \log(9/\delta)}{m \, \eigmin(W)}\Big)$. 
If $\max(1, 4c^2) dr \frac{\|\v\|_\inftyone^2 }{\eigmin(W)/t} \log(27/\delta) \leq {mt}$, we get that $\Delta \leq c\sqrt{\frac{dr\,\|\v\|_\inftyone^2 \log(9/\delta)}{m \, \eigmin(W)}}  \leq \frac12$. 

By setting $A + B = \cVih \cA \cVih$ and $A = \cI$ such that $\cE_1 = (A+B)^{-1} - B^{-1}$, in the Woodburry matrix inverse identity~\eqref{eq:woodbury_identity} (Lemma~\ref{lem:woodbury_identity}) we get that, with a probability of at least $1 - \delta$
\begin{align}
\|(A+B)^{-1} - A^{-1}\|_F &\leq \|A^{-1}\|_F\|B\|_F\|(A+B)^{-1}\|_F \\
\implies \|\cE_1\|_F \leq \Big\|(\cVih \cA \cVih)^\inv - \cI\Big\|_F &\leq \|\cI^\inv\|_F \|\cVih \cA \cVih - \cI \|_F \|(\cVih \cA \cVih)^\inv\|_F \\
&\leq 1 \cdot \Delta \cdot \frac3{2(1-\Delta)} 
\leq 3 \Delta
\leq 3c \sqrt{\frac{dr\,\|\v\|_\inftyone^2 \log(9/\delta)}{m \, \eigmin(W)}}
\end{align}	
Finally, setting $\delta \leftarrow \delta/3$ get us the desired result.
\end{proof}

	\begin{lemma}
	\label{lem:tail6}
	$\|(\cVih \cH ) (\oU Q)\|_F \leq \|\dv\|_F$ 
	and with a probability of at least $1 - \delta/3$
	\begin{align}
	\|\cE_2 (\oU Q)\|_F &\leq c (\min(\|\dv\|_F \frac{\|\v\|_\inftyone}{\sqrt{\eigmin(W)}},\|\dv\|_\inftyone )\sqrt{\frac{dr\,\log(15/\delta)}{m}} + \|\dv\|_\inftyone \frac{\|\v\|_\inftyone}{\sqrt{\eigmin(W)}} \frac{dr\,\log(15/\delta)}{m})
	\end{align}
\end{lemma}
\begin{proof}[Proof of Lemma~\ref{lem:tail6}]
	First we prove that the expected value $\bE[(\cVih \hcH ) (\oU Q)] = (\cVih \cH ) (\oU Q)$ is bounded.
	\begin{align}\label{eq:tail6_expect_ub1}
	\|(\cVih \cH ) (\oU Q)\|_F &= \max_{\|U\|_F = 1} \Ip{(\cVih \cH ) (\oU Q)}{U} \\
	&= \max_{\|U\|_F = 1} \sum_{i \in [t]} \Big\langle \oU Q \dvi (\vi)^\top \VVih, U \Big\rangle \\
	&= \max_{\|U\|_F = 1} \sum_{i \in [t]} \Big\langle \oU Q \dvi, U \VVih \vi \Big\rangle \\
	&\leq \max_{\|U\|_F = 1} \sqrt{\sum_{i \in [t]} \|\oU Q \dvi \|^2} \sqrt{\sum_{i \in [t]} \Ip{U \VVih \vi}{U \VVih \vi}} \\
	&\leq \max_{\|U\|_F = 1}  \|Q\|\sqrt{\sum_{i \in [t]} \|\dvi \|^2} \sqrt{\Ip{U \sum_{i \in [t]} \VVih \vi (\vi)^\top \VVih}{U}} \\
	&\leq \max_{\|U\|_F = 1} \|\dv\|_F \|U\|_F = \|\dv\|_F
	\end{align}
	where used the fact that $\Ip{AB}{C} = \Ip{A}{C B^\top}$ and $(\oU)^\top \oU = \Id$. 
	
	Let $\mathcal{S}_F = \{U \in \bR^{d \times r} \,|\, \|U\|_F = 1\}$ be the set of all real matrices of dimensions $d \times r$ with unit Frobenius norm. For $\epsilon \leq 1$, there exists an $\epsilon$-net, $N_\epsilon \subset \mathcal{S}_F$, of size $(1 + 2/\epsilon)^{dr}$ with respect to the Frobenius norm~\citep[Lemma 5.2]{vershynin2010introduction}. That is for any $U' \in \mathcal{S}_{F}$, there exists some $U \in N_\epsilon$ such that $\|U'-U\|_F \leq \epsilon$.
	
	Consider a $U \in N_\epsilon$, such that $\|U\|_F = 1$. Now we will prove with high-probability that $\big\langle (\cVih \cH) (\oU Q)(U) - \cVih(\sum_{i \in [t]} \Si \oU Q \dvi (\vi)^\top), U \big\rangle$ is small. Consider the the following quadratic form 
	\begin{align}
	\big\langle \cVih(\sum_{i \in [t]} \Si \oU Q \dvi (\vi)^\top), U \big\rangle &= \Big\langle \sum_{i \in [t]} \Si \oU Q \dvi (\vi)^\top \VVih, U\Big\rangle \\ 
	&= \sum_{i \in [t]} \frac1m \sum_{j \in [m]}  (\xij)^\top (\oU Q \dvi (\vi)^\top \VVih U^\top) \xij
	\end{align}
	where $\Si = \frac1m \sum_{j \in [m]}  \xij (\xij)^\top$ and $\xij \sim {\cal N}(0,{\mathbf I}_{d\times d}$) are i.i.d.~standard Gaussian random vectors and $\VV = \sum_{i \in [t]} \vi (\vi)^\top$ is rank-$r$ matrix. We will use Hanson-Wright inequality (Lemma~\ref{lem:hanson-wright}) to prove that the above quadratic form concentrates around its mean. Notice that the the expectation of $\big\langle \cVih(\sum_{i \in [t]} \Si \oU Q \dvi (\vi)^\top), U \big\rangle$ is $\big\langle\VVih \cH(U), U \big\rangle$.
	\begin{align}\label{eq:tail6_expect}
	\bE[\cVih(\sum_{i \in [t]} \Si \oU Q \dvi (\vi)^\top)] &=
	\cVih(\sum_{i \in [t]} \oU Q \dvi (\vi)^\top) = (\cVih \cH ) (\oU Q)\,.
	\end{align}
	We will also need the following bounds to apply the Hanson-Wright inequality. Recall that $\|\dv\|_\inftyone = \max_{i \in [t]} \|\dvi\|$ and $\|\v\|_\inftyone = \max_{i \in [t]} \|\vi\|$. Then,
	\begin{align}
	\max_{i \in [t]} \|\oU Q \dvi (\vi)^\top \VVih U^\top\| &\leq \max_{i \in [t]} \|\oU\| \|Q\| \|\dvi\| \max_{i \in [t]} \frac{\|\vi\|}{\sqrt{\eigmin(W)}}  \|U\| 
	\leq \|\dv\|_\inftyone \frac{\|\v\|_\inftyone}{\sqrt{\eigmin(W)}}
	\end{align}
	Also note that
	\begin{align}
	\sum_{i \in [t]} \|\oU Q \dvi (\vi)^\top \VVih U^\top\|_F^2 &= 
	\sum_{i \in [t]} \|\oU Q \dvi\|^2 \|U \VVih  \vi\|^2 \\
	&\leq (\sum_{i \in [t]} \|\oU Q \dvi\|^2)(\max_{i \in [t]} \|U \VVih  \vi\|^2) \\
	&\leq (\|Q\|^2 \sum_{i \in [t]} \|\dvi\|^2)(\max_{i \in [t]} \|U\|^2 \|\VVih\|^2  \|\vi\|^2) \\
	&\leq \|\dv\|_F^2 \frac{\|\v\|_\inftyone^2}{\eigmin(\VV)}
	\end{align}
	and
	\begin{align}
	\sum_{i \in [t]} \|\oU Q \dvi (\vi)^\top \VVih U^\top\|_F^2 &= 
	\sum_{i \in [t]} \|\oU Q \dvi\|^2 \|U \VVih  \vi\|^2 \\
	&\leq (\max_{i \in [t]} \|\oU Q \dvi\|^2) \tr(U \VVih \sum_{i \in [t]}  \vi (\vi)^\top \VVih U^\top) \\
	&\leq \|Q\| \max_{i \in [t]} \|\dvi\|^2 \|U\|_F^2 \\
	&= \|\dv\|_\inftyone^2\,.
	\end{align}
	Therefore, $\sum_{i \in [t]} \|\oU Q \dvi (\vi)^\top \VVih U^\top\|_F^2 \leq \min\{\|\dv\|_F^2 \frac{\|\v\|_\inftyone^2}{\eigmin(\VV)}, \|\dv\|_\inftyone^2\}$. For brevity, let $\cE_2 (U) =  \\ \cVih(\sum_{i \in [t]} \Si U \dvi (\vi)^\top) - (\cVih \cH) (U) $.
	Then by Hanson-Wright inequality (Lemma~\ref{lem:hanson-wright}), with probability at least $1 - \delta/|N_\epsilon|$
	\begin{align}
	\big| \big\langle \cE_2(\oU Q), U \big\rangle \big| = \big| \Big\langle \sum_{i \in [t]} \frac1m \sum_{j \in [m]} \xij (\xij)^\top \oU Q \dvi (\vi)^\top \VVih, U\Big\rangle - \Ip{(\cVih \cH) (\oU Q) }{U} \big| \leq \Delta_\epsilon
	\end{align}
	where $\Delta_\epsilon = c (\min(\|\dv\|_F \frac{\|\v\|_\inftyone}{\sqrt{\eigmin(W)}},\|\dv\|_\inftyone )\sqrt{\frac{\log(|N_\epsilon|/\delta)}{m}} + \|\dv\|_\inftyone \frac{\|\v\|_\inftyone}{\sqrt{\eigmin(W)}} \frac{\log(|N_\epsilon|/\delta)}{m})$. Taking union bound over all $U \in N_\epsilon$ implies that with probability at least $1 - \delta$
	\begin{align}
	\big| \big\langle \cE_2(U), U \big\rangle \big| \leq \Delta_\epsilon \, \text{ , \;for all } U \in N_\epsilon\,.
	\end{align}
	Let $U' \in \mathcal{S}_F \subset \bR^{d \times r}$ be the matrix ``parallel'' to $\cE_1$, that is $\|\cE_2(\oU Q)\|_F = \max_{\|\widetilde{U}\|_F = 1} \Ip{\cE_1(\oU Q)}{\widetilde{U}} = \Ip{\cE_2(\oU Q)}{U'}$. Then there exists some $U \in N_\epsilon$ such that $\|U'-U\|_F \leq \epsilon$. 
	\begin{align}
	\|\cE_2(\oU Q)\|_F = \Ip{\cE_2(\oU Q)}{U'} &= \Ip{\cE_2(\oU Q)}{U} + \Ip{\cE_2(\oU Q)}{U'-U} \\
	&\leq \Ip{\cE_1 (U)}{U} + \|\cE_2(\oU Q)\|_F \|U' - U\|_F \\
	&\leq \Ip{\cE_1 (U)}{U} + \epsilon \|\cE_2(\oU Q)\|_F 
	\end{align}
	Re-arranging and setting $\epsilon=1/2$, and $c \gets 2c$, we get
	\begin{align}
	\|\cE_2(\oU Q)\|_F  \leq \Delta_{\frac12} &= c (\min(\|\dv\|_F \frac{\|\v\|_\inftyone}{\sqrt{\eigmin(W)}},\|\dv\|_\inftyone )\sqrt{\frac{dr\,\log(5/\delta)}{m}} + \|\dv\|_\inftyone \frac{\|\v\|_\inftyone}{\sqrt{\eigmin(W)}} \frac{dr\,\log(5/\delta)}{m}) 
	\end{align}
	Finally setting $\delta \leftarrow \delta/3$ get us the desired result.

\end{proof}

	\begin{lemma}
	\label{lem:tail8}
	With a probability of at least $1 - \delta/3$
	\begin{align}
	\| \sum_{i \in [t]} \cVih (\zi (\vi)^\top)) \|_F &\leq 
	\Ord\Big(\sigma \sqrt{\frac{dr}{m} \log\Big(\frac{t}\delta\Big)  \log\Big(\frac{r}\delta\Big)} \Big)
	\end{align}
\end{lemma}
\begin{proof}[Proof of Lemma~\ref{lem:tail8}]

	Notice that $\zi$ (defined in Appendix~\ref{sec:analysis_rkr_logt}) is a Gaussian random vector of the following form
	\begin{align}
	\zi = \frac1m \sum_{j \in [m]} \epsij \xij = \frac1m \|\epsi\| \gi, \gi \sim \cN(0, \Id_{d \times d})
	\end{align}
	
	Using Hanson-Wright inequality (Lemma~\ref{lem:hanson-wright}, by setting $m \leftarrow 1$, $x_1 \leftarrow \epsi$, and $A_1 \leftarrow \Id_{m \times m}$) and taking union bound over all tasks, we get that, with probability of at least $1- \frac{\delta}2$
	\begin{align}
	\|\epsi\|^2 \leq \sigma^2 m (1 + c\sqrt{\frac{\log(\frac{2t}\delta)}{m}} + c\frac{\log(\frac{2t}\delta)}{m}) \leq 2c\,\sigma^2 m \log\Big(\frac{2t}\delta\Big) \, \text{,\;  for all } i \in [t]
	\end{align}
	where used the fact that $m \geq 1$ and $\log\Big(\frac{2t}\delta\Big) \geq 1$.
	
	Let $\hvi = \VVih \vi$, then
	\begin{align}
	\sum_{i \in [t]} \|\hvi\|^2 = \sum_{i \in [t]} \tr((\vi)^\top \VV^{-1} \vi )
	= \sum_{i \in [t]} \tr( \VV^{-1} \vi (\vi)^\top )
	= r
	\end{align}
	
	Notice that $\sum_{i \in [t]} \frac1m \|\epsi\| \gi \hvij$ is a Gaussian random vector of the following form
	\begin{align}
	\sum_{i \in [t]} \frac1m \|\epsi\| \gi \hvij = \frac1m \sqrt{\sum_{i \in [t]} \|\epsi\|^2 (\hvij)^2}\;\hgj \,, \hgj \sim \cN(0, \Id_{d \times d})
	\end{align}
	
	Using Hanson-Wright inequality (Lemma~\ref{lem:hanson-wright}, by setting $m \leftarrow 1$, $x_1 \leftarrow \hgj$, and $A_1 \leftarrow \Id_{d \times d}$) and taking union bound over all $j \in [r]$, we get that, with probability of at least $1- \frac{\delta}2$
	\begin{align}
	\|\hgj\|^2 \leq  d (1 + c\sqrt{\frac{\log(\frac{2r}\delta)}{d}} + c\frac{\log(\frac{2r}\delta)}{d}) \leq 2c d \log\Big(\frac{2r}\delta\Big) \, \text{,\;  for all } j \in [r]
	\end{align}
	where used the fact that $d \geq 1$ and $\log\Big(\frac{2r}\delta\Big) \geq 1$.
	
	Combining the above results and using union bound, we get that, with a probability of at least $1- {\delta}$,
	\begin{align}
	\Big\| \sum_{i \in [t]} \cVih (\zi (\vi)^\top)) \Big\|_F^2 = \Big\| \sum_{i \in [t]} \zi (\vi)^\top \VVih \Big\|_F^2 
	&= \Big\| \sum_{i \in [t]} \frac1m \|\epsi\| \gi (\hvi)^\top \Big\|_F^2 \\
	&= \sum_{j \in [r]} \Big\| \sum_{i \in [t]} \frac1m \|\epsi\| \gi \hvij \Big\|^2 \\
	&\leq \sum_{j \in [r]} \sum_{i \in [t]} \frac{\|\epsi\|^2}{m^2} (\hvij)^2 \|\hgj\|^2 \\
	&\leq \sum_{j \in [r]} \sum_{i \in [t]} \Ord\Big(\frac{m \sigma^2}{m^2} \log\Big(\frac{t}\delta\Big) \Big) (\hvij)^2 \Ord\Big(d \log\Big(\frac{r}\delta\Big) \Big) \\
	&\leq\Ord\Big(\frac{d \sigma^2}{m} \log\Big(\frac{t}\delta\Big)  \log\Big(\frac{r}\delta\Big) \Big) \sum_{i \in [t]} \|\hvi\|^2 \\
	&\leq \Ord\Big(\frac{\sigma^2 dr}{m} \log\Big(\frac{t}\delta\Big)  \log\Big(\frac{r}\delta\Big) \Big)\,.
	\end{align}
	Finally, we get the desired result by setting $\delta \leftarrow \delta/3$.
\end{proof}

\subsection{Analysis of QR decomposition}\label{sec:qr_r_inverse_logt_pf}
\begin{proof}[Proof of Lemma~\ref{lem:qr_r_inverse_logt}]
	\begin{align}
	\sigma_{\min}(R) \geq \min_{\|z\| = 1} {\|R z\|} 
	= \min_{\|z\| = 1} {\|\pU 
		R z\|} 
	&= \min_{\|z\| = 1} {\| 
		\hU z\|} \\
	&\geq \min_{\|z\| = 1} {\| 
		(\oU Q + \cV^\inv \cH(\oU Q) + F) z\|} \\
	&\geq \min_{\|z\| = 1} \sqrt{z^\top Q^\top Q z} - \|\cV^\inv \cH(\oU Q)\| - \|F\| \\
	&\geq \min_{\|z\| = 1} \sigma_{\min}(Q) - \|\cV^\inv \cH(\oU Q)\| - \|F\| \\
	&\geq \frac12 - \frac18 - \frac18 \geq \frac14
	\end{align}
	There fore $R$ is invertible and $\|R^\invert\| = (\sigma_{\min} (R))^{-1} \leq 4$
\end{proof}

\subsection{Analysis of shuffling and partitioning}
\label{sec:shuffling_rkr_logt_pf}
\begin{proof}[Proof of Lemma~\ref{lem:shuffling_rkr_logt}]
We will assume that the set of tasks $[t]$ is shuffled. We will prove that incoherence holds for the all subset $\cT_k = [1 + \frac{t(k-1)}{K},  \frac{tk}{K} ] $ of size $t/K$. Shuffling and $K$-way partitioning to get $\cT_k$ is equivalent to uniformly sampling without replacement $t/K$ elements from $[t]$. We prove that incoherrence holds for the first subset $\cT_1$, then this is equivalent to proving that incoherence holds for the $k$-th partition $\cT_k$ by symmetry.
Let the tasks sampled for $\cT_1$ without replacement be $\{i_l\}_{l=1}^{t/k}$, where $i_l$ is the $l$-th sample. 

Let $\mathcal{S}_F = \{z \in \bR^{r} \,|\, \|z\| = 1\}$ be the set of all real vectors of dimensions $r$ with unit Euclidean norm. For $\epsilon \leq 1$, there exists an $\epsilon$-net, $N_\epsilon \subset \mathcal{S}_F$, of size $(1 + 2/\epsilon)^{r}$ with respect to the Euclidean norm~\citep[Lemma 5.2]{vershynin2010introduction}. That is for any $z' \in \mathcal{S}_{F}$, there exists some $z \in N_\epsilon$ such that $\|z'-z\| \leq \epsilon$.

Consider a $z \in N_\epsilon$, such that $\|z\| = 1$. Now we will prove with high-probability that $z^\top (\sum_{l=1}^{t/K}  v^{*(i_l)} (v^{*(i_l)})^\top) z$ is approximately equal to $z^\top \bE[\sum_{l=1}^{t/K} v^{*(i_l)} (v^{*(i_l)})^\top] z$. Now consider the martingale $X_l$, such that $X_0 = 0$ and $X_l = X_{l-1} +z^\top (v^{*(i_l)} (v^{*(i_l)})^\top -  \bE[v^{*(i_l)} (v^{*(i_l)})^\top | X_0, \ldots, X_{l-1}]) z$, for all $l \in [t/K]$. Clearly this is a martginagle as $\bE[ X_{l} | X_0, \ldots, X_{l-1}] = 0$, for all $l \in [t/K]$. The maximum difference two consecutive steps is $\max_{l} |X_l - X_{l-1}| \leq 2 \|v^{*(i_l)}\|^2 \leq 2 \|\ov\|_\inftyone^2$. Therefore by Azuma-Hoeffding martingale inequality, 
\begin{align}
|\sum_{l=1}^{t/K} z^\top v^{*(i_l)} (v^{*(i_l)})^\top z -  z^\top \bE[\sum_{l=1}^{t/K}  v^{*(i_l)} (v^{*(i_l)})^\top ] z| = |X_{t/K} | \leq \sqrt{\frac{2t}{K} \|\v\|_\inftyone^4 \log(\frac{2|N_\epsilon|}{\delta})}
\end{align}
with a probability of at least $1 -\delta/|N_\epsilon|$.

For brevity, let $E = \sum_{l=1}^{t/K} v^{*(i_l)} (v^{*(i_l)})^\top  -  \bE[\sum_{l=1}^{t/K}  v^{*(i_l)} (v^{*(i_l)})^\top ]$. Notice that $E$ is a real symmetric matrix, therefore it has an eigen decomposition. Then, let $v' \in \mathcal{S} \subset \bR^{r}$ be the largest ``eigenvector'' of $E$, such that $(v')^\top E v' = \|E\| = \max_{\|\widetilde{v}\| = 1} \widetilde{v}^\top E \widetilde{v} = \max_{\|\widetilde{v}\| = \|\widetilde{v}'\|_F = 1} \widetilde{v}^\top E \widetilde{v}'$. Then there exists some $v \in N_\epsilon$ such that $\|v'-v\| \leq \epsilon$. 
\begin{align}
\|E\|_F = (v')^\top E v &= v^\top E v + (v'-v)^\top E v+  (v')^\top E (v' - v) \\
&\leq v^\top E v + \|v'-v\|\|E\| \|v\| +  \|v'\| \|E\| \|v' - v\|   \\
&\leq v^\top E v  + 2 \epsilon \|E\| 
\end{align}
Re-arranging and setting $\epsilon=1/4$, and $c \gets 2c$, we get
\begin{align}
\| \sum_{l=1}^{t/K} v^{*(i_l)} (v^{*(i_l)})^\top -  \bE[\sum_{l=1}^{t/K}  v^{*(i_l)} (v^{*(i_l)})^\top ] \| = \|E\| \leq \sqrt{\frac{2tr}{K} \|\v\|_\inftyone^4 \log(\frac{18}{\delta})} \leq \frac12  \eigmin(\bE[\sum_{l=1}^{t/K}  v^{*(i_l)} (v^{*(i_l)})^\top] ).
\end{align}
with probability at least $1 -\delta/k$, where the last inequality used the fact that $t \geq \Omega(\mu^2 r^3 K \log(1/\delta))$. Additionally note that $\bE[\sum_{l=1}^{t/k}  v^{*(i_l)} (v^{*(i_l)})^\top] = \frac1{K} \sum_{i=1}^{t}  \ovi (\ovi)^\top = \frac1K (\ov)^\top \ov$, Therefore
\begin{align}
\lambda_{r'}(\sum_{i \in \cT_k} \ovi(\ovi)^\top) = \frac1K \Theta(\lambda_{r'}((\ov)^\top \ov)) \text{ for all $r' \in [r]$}
\end{align}
where $\lambda_i(\cdot)$ is the $r'$-th largest eigenvalue matrix operator. 

\end{proof}

\section{Analysis of \altmins (Algorithm~\ref{alg:altselect}) with subset selection}
\label{sec:analysis_rkr_subset}
Initialized at $U$, the $k$-the step of alternating minimization-based \altmins (Algorithm~\ref{alg:altselect}) is:
\begin{eqnarray}
\cT_k &=& \big\{\, i\in [1 + (k-1)t/K, t k/K] \;|\;
\sigma_{\min}(U^\top \Si U) \geq 1/2 \text{ and }
\sigma_{\max}(U^\top \Si U) \leq 2
\,\big\}\\
\vi &\leftarrow&  (U^\top \Si U )^\inv( (U^\top \Si U^* ) \ovi + U^\top \zi)\;,\;\;\;\;\;\; \text{ for }i\in \cT_k  \label{eq:r_vupdate_subset}\\
\hU &\leftarrow&   \cA^\inv\Big( \,   \sum_{i\in \cT} \Si  U^* \ovi (\vi)^\top + \zi (\vi)^\top \,\Big) \;, \label{eq:r_uupdate_subset}\\
\pU & \leftarrow & \mathrm{QR}(\hU)\;,
\end{eqnarray}
where $\pU$ is the next iterate, $\Si_1 = \frac2m \sum_{j \in [1, m/2]} \xij (\xij)^\top$, $\Si_2 = \frac2m \sum_{j \in [1+m/2, m]} \xij (\xij)^\top$, $z^{(i)} \triangleq (1/m)\sum_{j\in[m]} \varepsilon_j^{(i)}x_j^{(i)}$ and ${\cal A}:{\mathbb R}^{d\times r} \to {\mathbb R}^{d\times r}$ is a  self-adjoint linear operator such that ${\cal A}(U)=  \sum_{i\in T} \Si  U \vi (\vi)^\top$. 
\\

\noindent{\color{black} \textbf{Remark: } Note the subset $\cT_k$, which we analyze, is slightly different from that of Algorithm $\ref{alg:altselect}$. This is done to save some $\mathrm{polylog}$ factors in the final error-bound and sample complexity. However, the analysis will remain almost the same even if eliminate the subset selection criterion, $\sigma_{\max}(U^\top \Si U) \geq 2$ for all $i \in \cT_k$.}

\begin{thm}\label{thm:altmin_subset}
	Let  there be $t$ linear regression tasks, each with $m$ samples satisfying Assumptions \ref{assume:linear_meta_problem} and \ref{assume:incoherence},
	and $K = \lceil \log_2(\frac{({\eigmin^*}/{\eigmax^*}) mt}{\mu dr^2}) \rceil$, $\|(\Id - \oU (\oU)^\top)\iU\|_F \leq \min \Big( \frac34, \Ord\Big(\sqrt{\frac{\eigmin^*}{\eigmax^*}  \frac1{\log(t/K)} } \Big)\Big)$, 
	$m \geq \Omega\Big( \Big(\SNRre\Big)^2 r^2 \log(\frac{t}{\delta}) + r^2 \log(\frac{K}{\delta}) + \log(\mu r) \Big)$, $t \geq \Omega(\mu^2 r^3 K \log(\frac{K}\delta))$
	and 
	$mt \geq \Omega\Big({\mu d r^2 K \frac{\eigmax^*}{\eigmin^*}}  \Big(\log(\frac{t}\delta) + \Big(\SNRre \Big)^2 \log^2(\frac{t}{\delta}) \log(\frac{rK}\delta) \Big)\Big)$.
	Then, for any $0 < \delta < 1$, after 
	$K$
	iterations, \altmins (Algorithm~\ref{alg:altselect}) returns an orthonormal matrix $U \in \bR^{d \times r}$, such that with a probability of at least $1 - \delta 
	$
	\begin{align}
	\frac1{\sqrt{r}} \|(\Id - \oU (\oU)^\top)U \|_F 
	&\leq \Ord\Big(\SNRre \sqrt{\frac{\mu d r K \log(\frac{t}{\delta}) \log(\frac{rK}\delta)}{mt}} \Big)
	\end{align}
	and the algorithm uses an additional memory of size $\Ord(d^2 r^2)$.
\end{thm}

A proof is in Section~\ref{sec:altmin_subset_pf}.

\subsection{Analysis} \label{sec:altmin_subset_pf}
First, in the following lemma, we prove that the task subset $\cT_k$ has similar properties as the full task partition $[1+t(k-1)/K, tk/K]$. 
\begin{lemma}[Subset selection]\label{lem:v_update_rkr_select_subset}
	If $m \geq \Omega( r + \log (\mu r))$ and 
	$t \geq \Omega(\mu^2\,r^2\,K \log(\frac{1}{\delta}))$,
	then with a probability of at least $1 - \delta/3$,
	\begin{align}
	&|\cT_k| = \Theta\Big(\frac{t}{K}\Big)\,,\;\; \text{\;, and \;\;} \|\ov\|_\inftyone^2 \leq \Ord\Big(\frac{\mu\, r}{|\cT_k|} \eigmin(\sum_{i \in \cT} \ovi (\ovi)^\top) \Big) \\
	&\eigmin(\sum_{i \in \cT} \ovi (\ovi)^\top) = \Theta(\eigmin(\sum_{i \in \cP_k} \ovi (\ovi)^\top)) \text{\;, and \;\;} \eigmax(\sum_{i \in \cT} \ovi (\ovi)^\top)  = \Theta(\eigmax(\sum_{i \in\cP_k} \ovi (\ovi)^\top)) 
	\end{align}
	where $\cP_k = [1+t(k-1)/K, tk/K]$ is the $k$-th $K$-way partition of $[t]$ after shuffling.
\end{lemma} 
A proof is in Section~\ref{sec:v_update_rkr_select_subset_pf}. Therefore, assuming that the above high-probability event holds, in the rest of the proof we can consider that $\cT_k$ is equivalent to $\cP_k$.

In the rest of the proof, when compared to the proof of Theorem~\ref{thm:altmin_logt}, only the following Lemma (corresponding to Lemma~\ref{lem:v_update_rkr_logt}) analyzing the $V$-update changes in its necessary condition.
\begin{lemma}\label{lem:v_update_rkr_subset}
	If $\|(\Id - \oU (\oU)^\top)U\|_F \leq \min \Big( \frac34, \Ord\Big(\sqrt{\frac{\eigmin^*}{\eigmax^*} \frac1{\log(t/K)} } \Big)\Big)$ and $m \geq \Omega\Big( \Big(\SNRre \Big)^2 r^2 \log(\frac{t}{K\delta}) + r \log(\frac{1}{\delta})\Big)$, 
	then with a probability of at least $1 - \delta/3$, 
	\begin{align}\label{eq:v_update_incoherence_rkr_subset}
	\|\vi\| \leq \Ord\Big(\mu\,\eigmin \Big) \text{\;, and } \eigmin^* \leq 2\eigmin
	\end{align}	
	and %
	\begin{align}
	\sqrt{\frac{rK}{t}}\frac{\|\dv\|_F}{\sqrt{\eigmin}} 
	&\leq \Ord\Big(\sqrt{ \frac{\log(\frac{t}{K\delta})}{\log(\frac{1}\delta)} }  \sqrt{\frac{\eigmax^*}{\eigmin^*}} \|(\Id - \oU (\oU)^\top) U\|_F + 
	\SNRre
	\sqrt{\frac{r^2 \log(\frac{t}{K\delta})}{m}} 
	\Big) \label{eq:v_update_hf_rkr_subset}
	\\
	\sqrt{\frac{r K}{t}} \frac{\|\dv\|_\inftyone}{\sqrt{\eigmin}} &\leq \Ord\Big(\sqrt{ \frac{\log(\frac{t}{K\delta})}{\log(\frac{1}\delta)} }  \|(\Id - \oU (\oU)^\top) U\| \sqrt{\frac{\mu r K}{t}} + 
	\SNRre
	\sqrt{\frac{r^2 K \log(\frac{t}{K\delta})}{mt}}
	\Big) \label{eq:v_update_hinfty_rkr_subset}
	\end{align}
\end{lemma}
A proof is in Section~\ref{sec:v_update_rkr_subset_pf}.
We omit the rest of the proof, as it is same as that of Theorem~\ref{thm:altmin_logt}.

\subsection{Analysis of task subset selection}
\label{sec:v_update_rkr_select_subset_pf}
\begin{proof}[Proof of Lemma~\ref{lem:v_update_rkr_select_subset} (Subset selection)]
	Let $\cP_k = [1 + (k-1)t/K, t k/K]$ and
	\begin{align}
	{\cT_k} &= \big\{\, i\in [1 + (k-1)t/K, t k/K] \;|\;
	\sigma_{\min}(U^\top \Si U) \geq 1/2 \,  \text{ and } \sigma_{\max}(U^\top \Si U) \leq 2\big\} \,.
	\end{align}
	For all $i \in \cP_k$, $X_i = \Ind(\sigma_{\rm \min}(U^\top \Si U) \geq 1/2  \text{ and } \sigma_{\max}(U^\top \Si U) \leq 2)$ be the indicator variable denoting whether index $i$ was select into the subset $\hcT$.
	
	By Lemma \ref{lem:gauss_sing_val} (by setting $a_j \gets 1$, $x_j \gets U^\top \xij$ for all $j\in[m]$, and $\delta \gets 1/4\mu r$) $X_i$ are i.i.d.~Bernoulli random variables with mean $p \geq 1 - \frac1{4\mu r}$, if $c \max\bigg(\sqrt{\frac{r\log(9) + \log(4\mu r)}{m}}, \frac{r\log(9) + \log(4\mu r)}{m} \bigg) \leq 1/2$, which is satisfied by $m \geq \Omega(r + \log (\mu r))$, for all $i \in \cP_k$.

	By Hoeffding inequality for Bernoulli random variables, with a probability of at least $1 - \delta/3$
	\begin{align}
	||\cT_k| - pt/K | %
	&= \big| \sum_{i \in \cP_k} X_i -  (1- \frac1{4\,\mu\,r}) \frac{t}{K} \big|
	\leq \frac{t}{K} \sqrt{\frac{K \log(\frac{3}{\delta})}{2t}} 
	\leq \frac{t}{K} \Ord\Big(\frac1{4\,\mu\,r} \Big)
	\end{align}
	where we used the fact that $t \geq \Omega(8 K \mu^2\,r^2\,\log(\frac{3}{\delta}))$. Therefore 
	\begin{align}
	\frac{t}K - |\cT_k| \leq \frac{t}{K} \Ord\Big(\frac1{2\,\mu\,r}\Big)  \text{\;, and \;\;}  |\cT_k| \leq \Theta\Big(\frac{t}{K}\Big)
	\end{align}
	where we used the fact that $\mu \geq 1$ and $r \geq 1$.

	\begin{align}
	\frac{r}{t}\big| z^\top(\sum_{i \in \cT_k} \ovi (\ovi)^\top)z - z^\top(\sum_{i \in \cP_k}\ovi (\ovi)^\top)z\big| &\leq \frac{r}{t} (t - \hatt) \|\ov\|_\inftyone^2 \leq \frac{r}{t} \Ord\Big(\frac{t}{2\mu\,r}\Big) \cdot \|\ov\|_\inftyone^2
	\leq \frac{\eigmin}2\,,
	\end{align}
	for all $z \in \bR^r$, where $\eigmin = \eigmin(\sum_{i \in \cP_k} \ovi (\ovi)^\top)$. Therefore 
	\begin{align}
	\eigmin(\sum_{i \in \cT} \ovi (\ovi)^\top) = \Theta(\eigmin(\sum_{i \in \cP_k} \ovi (\ovi)^\top)) \text{\;, and \;\;} \eigmax(\sum_{i \in \cT} \ovi (\ovi)^\top)  = \Theta(\eigmax(\sum_{i \in\cP_k} \ovi (\ovi)^\top))
	\end{align}
	Using approximate incoherence of the partition $\cP_k$ (Lemma~\ref{lem:shuffling_rkr_logt}) we get
	\begin{align}
	\|\ov\|_\inftyone^2 \leq \Ord\Big(\frac{\mu r K}{t}\Big) \eigmin(\sum_{i \in \cP_k} \ovi (\ovi)^\top) &= \Ord\Big(\frac{\mu r K}{t}\Big) \min_{\|z\| = 1}  z^\top(\sum_{i \in \cP_k} \ovi (\ovi)^\top)z \\
	&\leq \Ord\Big(\frac{\mu r K}{t}\Big) \min_{\|z\| = 1}  z^\top(\sum_{i \in \cT_k} \ovi (\ovi)^\top)z + \Ord\Big(\frac{\mu r K}{t}\Big) (\frac{t}{K} - |\cT_k|) \|\ov\|_\inftyone^2 \\
	&\leq \Ord\Big(\frac{\mu r K}{t}\Big) \eigmin(\sum_{i \in \cT_k} \ovi (\ovi)^\top) +  \frac12 \|\ov\|_\inftyone^2  \\
	\end{align}
	
	This implies that approximate incoherence holds for $\cT_k$, $\|\ov\|_\inftyone^2 \leq \Ord\Big(\frac{\mu r K}{t}\Big) \eigmin(\sum_{i \in \cT_k} \ovi (\ovi)^\top) \leq \Ord\big(\frac{\mu r}{|\cT_k|} \eigmin(\sum_{i \in \cT_k} \vi (\vi)^\top)\big)$.

\end{proof}

\subsection{Analysis of update on $V$}

\subsubsection{Proof of Lemma~\ref{lem:v_update_rkr_subset}}
\label{sec:v_update_rkr_subset_pf}
\begin{proof}[Proof of Lemma~\ref{lem:v_update_rkr_subset}]
	The proof is similar to that of Lemma~\ref{lem:v_update_rkr_logt}, but instead of using Lemma~\ref{lem:v_update_rkr_matrices_logt} to bound some linear operators, we use the definition of selected task subset $\cT_k$ and Lemma~\ref{lem:v_update_rkr_matrices_subset} to get that $\| (U^\top \Si U )^\inv \| \leq 2$ for all $i \in \cT_k$ and with a probability of at least $1 - \delta$,	\begin{align}
	\left.\begin{aligned}
	\|U^\top \Si U_\perp U_\perp^\top \oU \ovi\|
	&\leq \alpha \|U_\perp^\top \oU \ovi\big\| \text{, \; and \;\; } \\
	\big\| U^\top z^{(i)}  \,\big\| &\leq \sigma \alpha\,,
	\end{aligned}\right\rbrace
	\text{ for all $i \in \cT_k$}
	\end{align}
	where $\alpha = c \sqrt{\frac{r \log(10t/\delta)}{m}}$.
	We omit the rest of the proof, as it is same as that of Lemma~\ref{lem:v_update_rkr_logt}.
\end{proof}

Here we bound the linear operators in the $\vi$ update.
\begin{lemma}\label{lem:v_update_rkr_matrices_subset}
	With a probability of at least $1 - \delta$, the following are true for all $i \in [t]$
	\begin{align}
	\| U^\top \Si U_\perp (U_\perp)^\top \oU \ovi \|  &\leq \sqrt{\frac{2c r \log(10 t/\delta) }{m} }  \| U_\perp \oU \ovi\|  \text{, and } \\
	\big\| U^\top z^{(i)}  \,\big\| &\leq \sigma \sqrt{\frac{2c r \log(10 t/\delta) }{m} } 
	\end{align}
\end{lemma}
\begin{proof}%
Let $i \in [t]$. Let $b = (U_\perp)^\top \oU \ovi \in \bR^r$
	
Let $\mathcal{S} = \{v \in \bR^{r} \,|\, \|v\| = 1\}$ be the set of all real vectors of dimension $r$ with unit Euclidean norm. For $\epsilon \leq 1$, there exists an $\epsilon$-net, $N_\epsilon \subset \mathcal{S}$, of size $(1 + 2/\epsilon)^{r}$ with respect to the Euclidean norm~\citep[Lemma 5.2]{vershynin2010introduction}. That is for any $v' \in \mathcal{S}$, there exists some $v \in N_\epsilon$ such that $\|v'-v\|_F \leq \epsilon$.

Consider a $v \in N_\epsilon$, such that $\|v\|_F = 1$. Now we will prove with high-probability that $\big\langle (U^\top \Si U_\perp)v, b \big\rangle$ is small. Consider the the following quadratic form
\begin{align}
v^\top (U^\top \Si U_\perp) b = \frac1m \sum_{j \in [m]} v^\top (U^\top \xij (\xij)^\top U_\perp)  b  \overset{d}= \| b  \| \frac1m \sum_{j \in [m]} \widetilde{x}_j {g}_j
\end{align}
where $g_j \sim {\cal N}(0,1)$) are i.i.d.~standard Gaussian random variables and $\widetilde{x}_j = v^\top U^\top \xij \in \bR^d$. This follows from the fact that sets of columns of $U$ and $U_\perp$ forms an orthonormal basis.

Note that $g_j$ and $\widetilde{x}_j$ are independent, as $U$ and $U_\perp$ are orthogonal and $U^\top \Si U$, does not depend on $U_\perp \xij$.
We will use the properties of Gaussian random variables to prove that $\|\frac1m \sum_{j \in [m]} \widetilde{x}_j {g}_j\|$ concentrates around zero. Note that
\begin{align}
\frac1m \sum_{j \in [m]} \widetilde{x}_j {g}_j \overset{d}= \frac1m \|\widetilde{x}\| {g} \text{ ,\; where } g \sim \cN(0, 1)
\end{align}
Then with probability at least $1 - \delta/2t/|N_\epsilon|$, $|g|^2 \leq c \log(2t|N_\epsilon|/\delta)$. Additionally, by definition of $\cT_k$ we have
\begin{align}
\frac1m \|\widetilde{x}\|^2 =  \frac1m \sum_{j \in [m]} \widetilde{x}_j^2  = v^\top U^\top (\frac1m \sum_{j \in [m]} \xij (\xij)^\top) U v  \leq \sigma_{\max}(U^\top \Si U) \leq 2
\end{align}
Therefore
\begin{align}
v^\top (U^\top \Si U_\perp) b \leq \frac1{\sqrt{m}} \|b\| \sqrt{2c} \sqrt{\log(2t|N_\epsilon|/\delta)}
\end{align}
For brevity, let $e = (U^\top \Si U_\perp) b$. Let $v' \in \mathcal{S} \subset \bR^{r}$ be the unit vector parallel to $e$, such that $(v')^\top e = \|e\| = \max_{\|\widetilde{v}\| = 1} \widetilde{v}^\top e$. Then there exists some $v \in N_\epsilon$ such that $\|v'-v\| \leq \epsilon$. 
\begin{align}
\|e\| = (v')^\top e &= v^\top e + (v'-v)^\top e \leq v^\top e + \|v'-v\|\| e\| \leq v^\top e + \epsilon \| e\| 
\end{align}
Re-arranging and setting $\epsilon=1/2$, and $c \gets 2c$, we get
\begin{align}
\| (U^\top \Si U_\perp) b \| \leq \|b\| \sqrt{\frac{2c r \log(10 t/\delta) }{m} } , \text{ with a probability of at least $1 - \delta/2t$}
\end{align}
Using similar arguments we can also prove that with a probability of at least $1 - \delta$
\begin{align}
\| U^\top \zi \| = \| \frac1m U^\top \xij \epsij \| \leq \sigma \sqrt{\frac{2c r \log(10 t/\delta) }{m} } , \text{ with a probability of at least $1 - \delta/2t$}
\end{align}
Finally %
taking the union bound over the two bounds over all the tasks in $\cT$ gets us the desired result.
\end{proof}

\section{Corollaries of known results}
\label{sec:past_corollaries}

\begin{thm}[Theorem 3, Tripuraneni et al. 2020]\label{thm:svd_mom}
	Let  there be $t$ linear regression tasks, each with $m$ samples satisfying Assumptions \ref{assume:linear_meta_problem} and \ref{assume:incoherence}, and
	\begin{align}
	mt &\geq \widetilde\Omega \Big(\frac{\eigmax^*}{\eigmin^*} {\mu dr} + \Big( \SNRre \Big)^4 dr^2\Big)
	\end{align}
	then with a high probability of at least $1 - \Ord((mt)^{-100})$, Method-of-Moments~\citep[Algorithm 1]{tripuraneni2020provable} outputs an orthonormal matrix $U \in \bR^{d \times r}$ such that
	\begin{align}
	\|(\Id - \oU (\oU)^\top) U\|_2 &\leq \tilde\Ord\Big(\sqrt{\frac{\eigmax^*}{\eigmin^*} \frac{\mu d r}{mt}} + \Big(\SNRre \Big)^2 \sqrt{\frac{dr^2}{mt}} \Big) 
	\end{align}
	and 
	\begin{align}
	\|(\Id - \oU (\oU)^\top) U\|_F &\leq \tilde\Ord\Big(\sqrt{\frac{\eigmax^*}{\eigmin^*} \frac{\mu d r^2}{mt}} + \Big(\SNRre \Big)^2 \sqrt{\frac{dr^3}{mt}} \Big) \,.
	\end{align}
\end{thm}
\begin{proof}
From the details of the proof of Theorem 3 in \citep{tripuraneni2020provable} we can derive that, with a high probability of at least $1 - \Ord((mt)^{-100})$,
\begin{align}
&\|(\Id - \oU (\oU)^\top) U\|_2 \\
&\leq \tilde\Ord\Big(\sqrt{\frac{dr^2 \tr(W^*) \|\ov\|_\inftyone^2}{\eigmin^{*2}\, m t^2}} + \frac{dr \|\ov\|_\inftyone^2}{\eigmin^*\, mt} + \sigma \Big(\sqrt{\frac{dr^2 \tr(W^*)}{\eigmin^{*2}\, mt^2 }} + \frac{d r \|\ov\|_\inftyone}{{\eigmin^*}\, mt} \Big) + \sigma^2 \Big( \sqrt{\frac{d r^2 }{\eigmin^{*2}\,mt}} + \frac{dr}{\eigmin^*\,mt} \Big) \Big) \\
&\leq \tilde\Ord\Big(\sqrt{\frac{\eigmax^*}{\eigmin^*} \frac{\mu d r}{mt}} + \frac{\mu d r}{mt} + \SNRre \bigg(\sqrt{\frac{\eigmax^*}{\eigmin^*} \frac{d r}{mt}} + \frac{\sqrt{\mu}\,dr }{{mt}} \bigg) + \Big(\SNRre \Big)^2 \Big( \sqrt{\frac{dr^2}{mt}} + \frac{dr}{mt} \Big) \\
&\leq \tilde\Ord\Big(\sqrt{\frac{\eigmax^*}{\eigmin^*} \frac{\mu d r}{mt}} + \SNRre \sqrt{\frac{\eigmax^*}{\eigmin^*} \frac{d r}{mt}} + \Big(\SNRre \Big)^2 \sqrt{\frac{dr^2}{mt}} \Big) \\
&\leq \tilde\Ord\Big(\sqrt{\frac{\eigmax^*}{\eigmin^*} \frac{\mu d r}{mt}} + \Big(\SNRre \Big)^2 \sqrt{\frac{dr^2}{mt}} \Big) 
\end{align}
where $\|\v\|_\inftyone = \max_{i \in [t]} \|\vi\|$, and the second-last inequality uses the fact that $mt \geq \widetilde\Omega(\mu dr)$ and last inequality uses the fact that $\frac{\eigmax^*}{\eigmin^*}  \leq \mu r$. Additionally we require that
\begin{align}
mt &\geq \widetilde\Omega \Big(\frac{\eigmax^*}{\eigmin^*} {\mu dr} + \Big( \SNRre \Big)^4 dr^2\Big)
\end{align}
\end{proof}

\begin{thm}\citep[Theorem 5]{tripuraneni2020provable}  
	\label{thm:meta_learn_basis_lb}    
	Let $r\leq d/2$ and $mt\geq r(d-r)$, then for all $V^*$, w.p.~$\geq 1/2${
		\small
		\begin{eqnarray*}
			\inf_{\widehat{U}} \sup_{U\in{\rm Gr}_{r,d} }  \frac{\| ( {\mathbf I}-U^*(U^*)^\top ) \widehat{U} \|_F}{\sqrt{r}} \;\geq\; \Omega\Big(  \Big(\frac{\eigmin^*}{\eigmax^*} \frac{\sigma}{\sqrt{\eigmin^*} }\Big) \sqrt{ \frac{d\,r}{m\,t}} \Big) \;,
		\end{eqnarray*}
	}where $G_{r,d}$ is the Grassmannian manifold of $r$-dimensional subspaces in ${\mathbb R}^d$, the infimum for $\widehat{U}$ is taken over the set of all measurable functions that takes $mt$ samples in total from the model in Section~\ref{sec:problem} satisfying Assumption~\ref{assume:linear_meta_problem} and \ref{assume:incoherence}. 
\end{thm}	
\begin{proof}
The proof is very similar to that of Theorem 5 of \citet{tripuraneni2020provable}. The main difference is that instead of lower bounding error in spectral norm we have to bound the distance in the Frobenius norm. However, the rest of the details are almost the same, hence we omit a full proof.
\end{proof}

\section{Technical Lemmas}
\label{sec:technical_lemmas}

This section contains some technical lemmas used in this paper.

\begin{lemma}\label{lem:eig_prod}
For a real matrix $A \in \bR^{m \times n}$ and a real symmetric positive semi-definite (PSD) matrix $B \in \bR^{n \times n}$, the following holds true: $\sigma_{\min}^2(A) \lambda_{\min}(B) \leq \lambda_{\min}(ABA^\top)$, where $\sigma_{\min}(\cdot)$ and $\lambda_{\min}(\cdot)$ represents the minimum singular value and minimum eigenvalue operators respectively.
\end{lemma}
\begin{proof} The proof directly follows from the definitons of $\sigma_{\min}$ and $\lambda_{\min}$. Since $B$ is a PSD matrix, therefore $ABA^\top$ is also PSD, i.e.~$\lambda_{\min}(ABA^\top) \geq 0$. This is because since $B$ is PSD, it has a PSD matrix square root $B^{1/2}$ such that $B = (B^{1/2})^\top B^{1/2}$ and $B^{1/2}$ is PSD. Then
\begin{align}
z^\top A B A^\top z &= z^\top A (B^{1/2})^\top B^{1/2} A^\top z = \|B^{1/2} A^\top z\|^2 \geq 0
\end{align}	
First assume that $\sigma_{\min} (A) > 0$, then
\begin{align}
\lambda_{\min}(ABA^\top) &= \min_{\|z\| = 1} z^\top ABA^\top z \\
&= \sigma_{\min}^2(A) \min_{\|z\| = 1} (\frac{A^\top z}{\sigma_{\min}(A)})^\top B (\frac{A^\top z}{\sigma_{\min}(A)}) \\
&\geq \sigma_{\min}^2(A) \min_{1\leq \|z\| \leq \frac{\sigma_{\max}(A)}{\sigma_{\min}(A)} } z^\top B z \\
&\geq \sigma_{\min}^2(A) \min_{\|z\| = 1} z^\top B z \\
&= \sigma_{\min}^2(A) \lambda_{\min}(B)
\end{align}
The second last inequality above follows from the fact that $B$ is a PSD matrix,i.e.~$\min_{\|z\| = 1} z^\top B z = \lambda_{\min}(B) \geq 0$. Secondly if $\sigma_{\min} (A) = 0$, then $A$ is rank deficient and hence $ABA^\top$ is also rank deficient, i.e.~$\lambda_{\min} (ABA^\top) = 0$. Therefore $\lambda_{\min} (ABA^\top) = 0 = \sigma_{\min}^2(A) \lambda_{\min}(B)$.
\end{proof}

\begin{lemma}[Weyl's inequality~\citep{amir1956extreme}]\label{lem:weyls}
For three real $r$-rank matrices, satisfying  $A - B = C$, Weyl's inequality~\citep[Theorem 3.6]{amir1956extreme}, tells that 
\begin{align}
\sigma_{k}(A) - \sigma_{k}(B) \leq  \|C\| \text{\,,  for all $k \in [r]$}
\end{align}
where $\sigma_k(\cdot)$ is the $k$-th largest singular value operator.
\end{lemma}

\begin{lemma}[a variant of Woodburry matrix identity~\cite{henderson1981deriving}]
\label{lem:woodbury_identity}
For linear operators $A$ and $B$ such that $A$ and $A+B$ are invertible, then
\begin{align}
(A + B)^{-1} - A^{-1} = - A^{-1}B(A + B)^{-1} \label{eq:woodbury_identity}
\end{align} 

\end{lemma}

\begin{lemma}\label{lem:distance-relations}
	Let $U \in \bR^{d \times r}$ and $\oU \in \bR^{d \times r}$ be two orthonormal matrices. Let $\{\sin \theta_j (U, \oU) \}_{j=1}^r$ be the singular values of $(\oU)^\top U$. Then following are true.
	\begin{align}
	\|U - \oU (\oU)^\top U \|_F &\geq \|\Id -  (\oU)^\top U\|_F \text{ , } \\
	\|U - \oU (\oU)^\top U \|_F &\geq r - \|(\oU)^\top U\|_F^2 \geq  \sum_{k \in [r]} \sin^2 \theta_k(U, \oU) \text{, } \\
	\|(\Id -  \oU (\oU)^\top) U\| = \|(\oU_\perp)^\top U\| &= \|U_\perp^\top \oU\| = \|(\Id -  U (U)^\top) \oU\| \text{, } \\
	\|(\Id -  \oU (\oU)^\top) U\|_F = \|(\oU_\perp)^\top U\|_F &= \|U_\perp^\top \oU\|_F = \|(\Id -  U (U)^\top) \oU\|_F \text{, and } \\
	\sigma_r((\oU)^\top U) &\geq \sqrt{1 - \|(\Id - \oU (\oU)^\top) U\|}
	\end{align}
\end{lemma}
\begin{proof}
	\begin{align}
	\|U - \oU (\oU)^\top U \|_F^2 &= \Ip{U - \oU (\oU)^\top U}{U - \oU (\oU)^\top U} \\
	&= \Ip{U}{U} - 2 \Ip{\oU (\oU)^\top U }{U} + \Ip{\oU (\oU)^\top U }{\oU (\oU)^\top U} \\
	&= r - 2 \tr(((\oU)^\top U)^\top ((\oU)^\top U)) + \tr(((\oU)^\top U)^\top ((\oU)^\top U)) \\
	&= r -  \tr(((\oU)^\top U)^\top ((\oU)^\top U)) \\
	&= r - \sum_{k \in [r]} \cos^2 \theta_k(U, \oU) 
	= \sum_{k \in [r]} \sin^2 \theta_k(U, \oU) 
	\geq \sin^2 \theta_1(U, \oU) \\
	&\geq \sum_{k \in [r]} (1-\cos^2 \theta_k(U, \oU)) \\
	&\geq \sum_{k \in [r]} (1-\cos \theta_k(U, \oU))^2 \\
	&= \|\Id - (\oU)^\top U \|_F^2
	\end{align}
	
	\begin{align}
    \|U_\perp^\top \oU\| = \sigma_{\max}(U_\perp^\top \oU)  &= \sqrt{\lambda_{\max}((\oU)^\top U_\perp U_\perp^\top \oU)} \\
    &= \sqrt{\lambda_{\max}((\oU)^\top U_\perp U_\perp^\top U_\perp U_\perp^\top \oU)} = \|U_\perp U_\perp^\top \oU\| = \|(\Id - U U^\top) \oU\|
	\end{align}
	Note that for $\|z\| = 1$
	\begin{align}
	&1 = z^\top U^\top U z = z^\top U^\top \oU (\oU)^\top U z + z^\top U^\top \oU_\perp (\oU_\perp)^\top U z \\
	\implies &1 - z^\top U^\top \oU (\oU)^\top U z = z^\top U^\top \oU_\perp (\oU_\perp)^\top U z \\
	\implies &1 - \min_{\|z\| = 1} z^\top U^\top \oU (\oU)^\top U z = \max_{\|z\| = 1} z^\top U^\top \oU_\perp (\oU_\perp)^\top U z \\
	\implies &1 - \sigma_{\min}^2 ((\oU)^\top U) = \|(\oU_\perp)^\top U\|^2
	\end{align}
	Therefore 
	\begin{align}
	\sigma^2_{\min}(U^\top \oU) + \|U_\perp^\top \oU\|^2 = 1 = \sigma^2_{\min}((\oU)^\top U) + \|(\oU_\perp)^\top U\|^2 \implies \|U_\perp^\top \oU\| = \|(\oU_\perp)^\top U\|
	\end{align}
	Rest of the equality can be obtained in a similar fashion using the above two relations.
	
	\begin{align}
	\|U_\perp^\top \oU\|_F^2 = \tr((\oU)^\top U_\perp U_\perp^\top \oU) &= \tr((\oU)^\top (\Id - U U^\top) \oU) \\
	&= \tr((\oU)^\top (\Id - U U^\top)^2 \oU) \\
	&= \|(\Id - U U^\top) \oU\|_F^2 \\
	&= \|(\Id - \oU (\oU)^\top) U\|_F^2 = \|(\oU_\perp)^\top U\|_F^2 
	\end{align}
	
	Let $E = (\Id - \oU (\oU)^\top) U$ and $Q =(\oU)^\top U$. Then $U^\top E = \Id - Q^\top Q$. Then by Weyl's inequality (Lemma~\ref{lem:weyls}, by setting $A \leftarrow \Id$, $B \leftarrow Q^\top Q$, and $C \leftarrow U^\top E$) we get that
	\begin{align}
	1 - \sigma_{r}(Q)^2 = \sigma_{r}(\Id) - \sigma_{r}(Q^\top Q) \leq  \|U^\top E\| \leq \|U\|\|E\| \leq \|(\Id - \oU (\oU)^\top) U\|
	\end{align}
	This implies that $\sigma_r((\oU)^\top U) \geq \sqrt{1 - \|(\Id - \oU (\oU)^\top) U\|}$
\end{proof}

\begin{lemma}[Hanson-Wright inequality, Theorem 6.2.1~\cite{vershynin2018high}] Let $x_1, \ldots, x_m \sim {\cal N}(0,{\mathbf I}_{d\times d})$ be $m$ i.i.d.~standard isotropic Gaussian random vectors of dimension $d$. Then, for some universal constant $c \geq 0$, the following holds true with a probability of at least $1 - \delta$.
	\begin{equation}
	\bigg|\frac1m \sum_{j=1}^m x_j^{\top} A_j x_j - \frac1m \sum_{j=1}^m \tr{A_j} \bigg| \leq c \max\bigg(\sqrt{\sum_{j=1}^m \|A_j\|^2_F \frac{ \log(1/\delta)}{m^2}}, \max_{j=1,\ldots,n}\|A_j\|_2 \frac{\log(1/\delta)}{m} \bigg)
	\end{equation}
	\label{lem:hanson-wright}
\end{lemma}

\begin{lemma} Let $x_1, \ldots, x_m \sim {\cal N}(0,{\mathbf I}_{d\times d})$ be $m$ i.i.d.~standard isotropic Gaussian random vectors of dimension $d$. Then, for some universal constant $c \geq 0$, the following holds true with a probability of at least $1 - \delta$.
	\begin{equation}
	\bigg|\frac1m \sum_{j=1}^m a^\top (x_j x_j^{\top}) b- a^\top b \bigg| \leq c \|a\|\|b\| \max\bigg(\sqrt{\frac{ \log(1/\delta)}{m}}, \frac{\log(1/\delta)}{m} \bigg)
	\end{equation}
	\label{lem:gauss_pseudo_inner}
\end{lemma}
\begin{proof}
	First notice that $a^\top (x_j x_j^{\top}) b = \tr(a^\top (x_j x_j^{\top}) b) = \tr(x_j^\top b a^\top x_j) = x_j^\top b a^\top x_j$ and $a^\top b = \tr(b a^\top)$. Then desired result follows from Lemma~\ref{lem:hanson-wright}, by setting $A_j = b a^\top$. .
\end{proof}

\begin{lemma} Let $x_1, \ldots, x_m \sim {\cal N}(0,{\mathbf I}_{d\times d})$ be $m$ i.i.d.~standard isotropic Gaussian random vectors of dimension $d$. Then, for some universal constant $c \geq 0$, the following holds true with a probability of at least $1 - \delta$.
	\begin{equation}
	\bigg\|\frac1m \sum_{j=1}^m a_j x_j x_j^{\top} - \frac1m \sum_{j=1}^m a_j \Id \bigg\| \leq c \max\bigg(\frac{\|a\|_2}{\sqrt{m}} \sqrt{\frac{d \log(9) + \log(1/\delta)}{m}}, \|a\|_\infty \frac{d \log(9) + \log(1/\delta)}{m} \bigg)
	\end{equation}
	\label{lem:gauss_sing_val}
\end{lemma}
\begin{proof}
	For $\epsilon \leq 1$, consider a unit vector $u \in N_\epsilon$ from the $\epsilon$-net of size $|N_\epsilon| = (1 + 2/\epsilon)^d$, of the sphere $\bS^{d-1}$~\citep[Lemma 5.2]{vershynin2010introduction}. That is for any $u' \in \bS^{d-1}$, there exists some $u \in N_\epsilon$ such that $\|u'-u\| \leq \epsilon$.
	
	Now we will prove a concentration for $\frac1m \sum_{j=1}^m a_j u^\top x_j x_j^{\top} u - \frac1m \sum_{j=1}^m a_j$.
	Notice that, $a_j u^\top (x_j x_j^{\top}) u = a \tr(u^\top (x_j x_j^{\top}) u) = a_j \tr(x_j^\top u u^\top x_j) =  x_j^\top (a_j u u^\top) x_j$ and $\tr(a_j u u^\top) = a_j$. Then, by Hanson-Wright inequality (Lemma~\ref{lem:hanson-wright}), for some universal constant $c \geq 0$, the following holds true with a probability of at least $1 - \delta'$.
	\begin{equation}
	\bigg| \frac1m \sum_{j=1}^m a_j u^\top x_j x_j^{\top} u - \frac1m \sum_{j=1}^m a_j \bigg| \leq c \max\bigg(\frac{\|a\|_2}{\sqrt{m}} \sqrt{\frac{ \log(1/\delta')}{m}}, \|a\|_\infty \frac{ \log(1/\delta')}{m} \bigg)
	\end{equation}
	This implies that, through union bound, for the matrix $A' = \frac1m \sum_{j=1}^m a_j x_j x_j^{\top} - \frac1m \sum_{j=1}^m a_j \Id$ the following holds true with probability at least $1 - \delta$
	\begin{equation}
	u^\top A' u \leq c \max\bigg(\frac{\|a\|_2}{\sqrt{m}} \sqrt{\frac{ \log(|N_\epsilon|/\delta)}{m}}, \|a\|_\infty \frac{\log(|N_\epsilon|/\delta)}{m} \bigg)\,, \;\;\; \text{any $u \in N_\epsilon$}
	\end{equation}
	Let $u' \in \bS^{d-1}$ be the top singular-value of $A'$, then there exists some $u \in N_\epsilon$ such that $\|u'-u\| \leq \epsilon$.
	\begin{align}
	\sigma_{\max}(A') = (u')^\top A' u' &= (u' - u)^\top A' u' + u^\top A' u + u^\top A' u \\
	&\leq \|u' - u\|\sigma_{\max}(A') \|u'\| + \|u\|\sigma_{\max}(A') \|u' - u\| + u^\top A' u \\
	\end{align}
	Re-arranging and setting $\epsilon=1/4$ and setting $c \gets 2c$, we get
	\begin{align}
	\sigma_{\max}(A') &\leq \frac{u^\top A' u}{1 - 2\epsilon} \leq 2c \max\bigg(\frac{\|a\|_2}{\sqrt{m}} \sqrt{\frac{d \log(9) + \log(1/\delta)}{m}}, \|a\|_\infty \frac{d \log(9) + \log(1/\delta)}{m} \bigg)
	\end{align}
\end{proof}

\end{document}